\documentclass{article} % For LaTeX2e
\let\standardaddcontentsline\addcontentsline
\usepackage{iclr2026_conference,times}
\let\addcontentsline\standardaddcontentsline

\usepackage{amsmath,amsfonts,bm}

\def\eqref#1{equation~\ref{#1}}
\def\1{\bm{1}}

\DeclareMathAlphabet{\mathsfit}{\encodingdefault}{\sfdefault}{m}{sl}
\SetMathAlphabet{\mathsfit}{bold}{\encodingdefault}{\sfdefault}{bx}{n}

\usepackage{hyperref}
\usepackage{etoc}
\usepackage{url}
\usepackage{enumitem}
\usepackage{amsmath}
\usepackage{amssymb}
\usepackage{amsthm}
\newtheorem{theorem}{Theorem}
\newtheorem{proposition}{Proposition}

\usepackage{colortbl}
\usepackage{xcolor}
\definecolor{baselinegray}{gray}{0.9}
\usepackage{graphicx}
\usepackage{subcaption}
\usepackage{multirow}
\usepackage{booktabs}
\usepackage{comment}

\newcommand{\old}{\mathrm{old}}
\newcommand{\dd}{\mathrm{d}}

\title{FlowAWR: Online Adaptive Flow Reinforcement via Advantage-Weighted Rectification}

\iclrfinalcopy

\author{%
Zheming Fu$^{1,2,3}$, \hspace{.3em}
Ruizhe He$^{2,4}$, \hspace{.3em}
Wei Shang$^{1}$, \hspace{.3em}
Xiaoxiao Ma$^{2,5}$, \hspace{.3em}
Lianrui Mu$^{2,4}$, \hspace{.3em}\\
\textbf{Lei Wang$^{1}$}\thanks{{} {} Corresponding authors.}, \hspace{.3em}
\textbf{Chang Liu$^{3}$}\footnotemark[1], \hspace{.3em}
\textbf{Siming Fu$^{2}$}\footnotemark[1]\\
[1ex]
$^{1}$ Beihang University, 
$^{2}$ Joy Future Academy, 
$^{3}$ Zhongguancun Academy, 
$^{4}$ Zhejiang University, \\
$^{5}$ University of Science and Technology of China
}

\begin{document}
\etocdepthtag.toc{main}

\maketitle
\lhead{}  % Arxiv release: uncomment to clear the ICLR header text.

% =====================================================
% Body migrated from FlowAWR/main.tex.
%   * \cite{} -> \citep{} (natbib parenthetical).
%   * \keywords{} removed (ICLR has no keyword field).
%   * llncs-only commands (\titlerunning, \authorrunning, \institute, \inst, \orcidlink) removed.
%   * Figures resolve through Figure/ at the iclr2026/ directory root.
% =====================================================

\vspace{-0.3in}
\noindent
\begin{minipage}{\textwidth}
    \captionsetup{type=figure,justification=justified,singlelinecheck=false}
    \centering
    \includegraphics[width=1.0\linewidth,trim=12bp 11bp 13bp 9bp,clip]{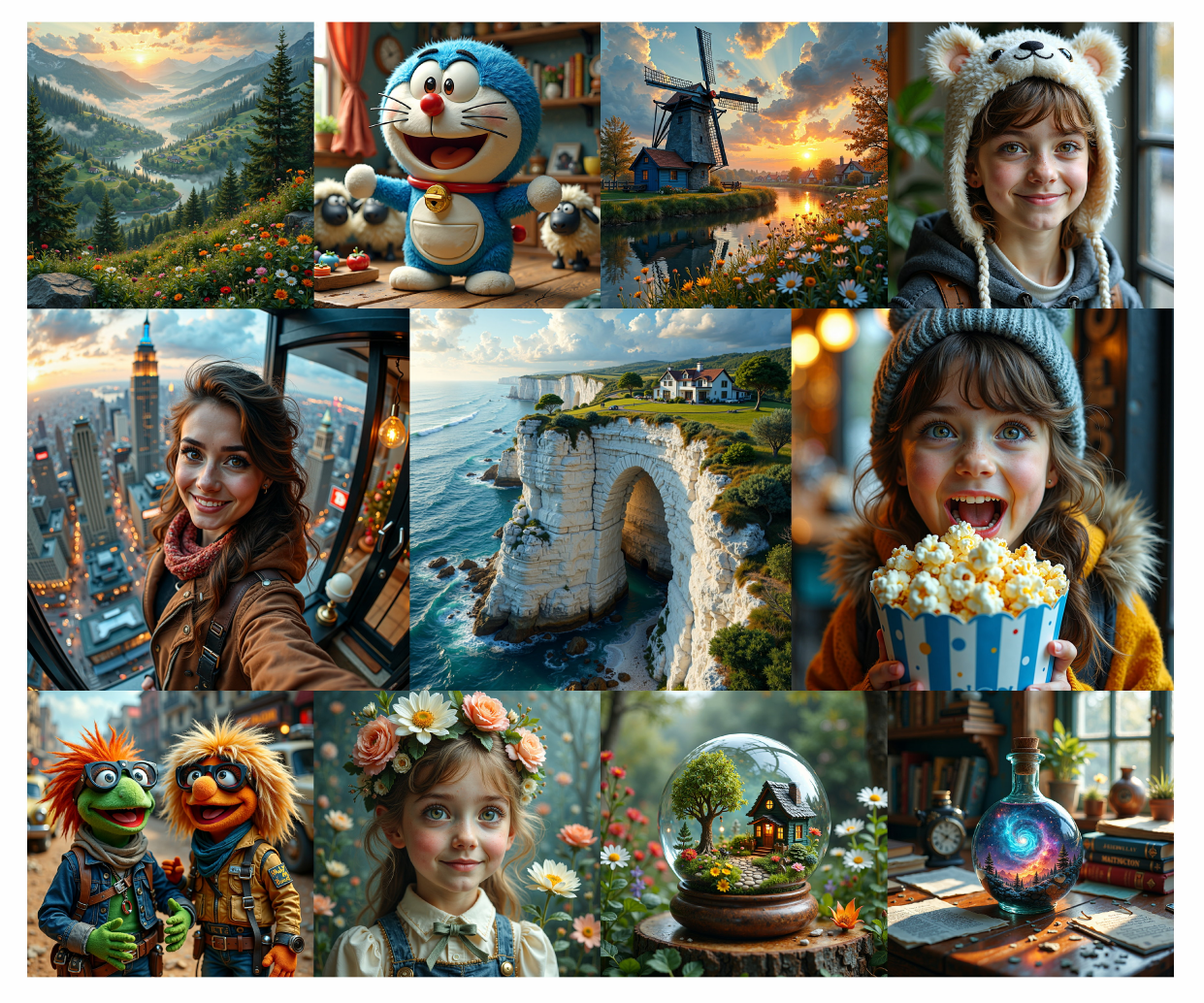}\par
    \caption{Samples generated by FLUX.1-dev optimized with FlowAWR using HPSv3 as the reward. Prompts are drawn from the HPDv2 and PickScore test sets.}
    \label{fig:teaser}
\end{minipage}\par

\begin{abstract}
  Aligning generative flow models on continuous spaces via online reinforcement learning is constrained by intractable trajectory likelihoods. 
  Existing density-approximated policy gradient methods rely on stochastic SDE samplers to construct tractable transition kernels, which introduce training-inference inconsistencies and necessitates Classifier-Free Guidance (CFG).
  DiffusionNFT bypasses trajectory likelihoods through implicit velocity regression, but uses fixed residual coefficients in its branch targets, with rewards controlling their relative loss weights.
  We propose \textit{Flow Advantage-Weighted Rectification} (\textbf{FlowAWR}), a paradigm that recasts continuous generative policy optimization as supervised regression toward a theoretically optimal velocity field.
  Starting from the optimal policy of a KL-constrained reward maximization, FlowAWR derives the optimal velocity field that admits a magnitude-aware, advantage-weighted rectification form, yielding SDE-free optimization.
  In comparative evaluations on SD3.5-Medium, FlowAWR achieves improved alignment performance alongside a 2$\times$ to 5$\times$ convergence acceleration over DiffusionNFT (e.g., reaching a 24.12 PickScore in 1.2k steps, versus 23.82 in 2.0k steps for DiffusionNFT and 23.50 in $>$4k steps for FlowGRPO). Under multi-reward constraints, FlowAWR sustains quality, satisfying structural rules with stable out-of-domain performance.
  Code at \href{https://github.com/Fuzm7/FlowAWR}{FlowAWR}.
\end{abstract}

\section{Introduction}
\label{sec:intro}

The empirical success of Reinforcement Learning (RL) in the post-training of Large Language Models (LLMs)~\citep{ziegler2019fine, achiam2023gpt, comanici2025gemini, yang2025qwen3} has motivated its extension to continuous generative spaces.
Within this domain, Flow Matching~\citep{lipman2022flow, liu2023flow} has emerged as a prevalent paradigm, offering a deterministic trajectory framework for modern diffusion models~\citep{esser2024scaling, labs2025flux, seedream2025seedream, wu2025qwen}.
However, adapting RL to this continuous space introduces a structural constraint: the approximation of trajectory likelihood. Unlike AR models that compute exact discrete token probabilities, continuous flow models inherently lack tractable step-wise transition densities~\citep{song2020scorebased}. 
Ranging from Direct Preference Optimization (DPO) approaches~\citep{wallace2024diffusion, yang2024using} to recent online policy gradient approaches~\citep{liu2025flow, xue2025dancegrpo}, this likelihood bottleneck consistently compromises the computational efficiency and theoretical guarantees of RL on continuous spaces.

To circumvent this likelihood barrier, a prevalent strategy discretizes the continuous generative process into a multi-step Markov Decision Process (MDP)~\citep{black2023training, fan2023dpok}. Building upon this, recent GRPO-based methods~\citep{liu2025flow, xue2025dancegrpo} inject stochasticity via SDE samplers to construct tractable Gaussian transition kernels, enabling the evaluation of policy gradients. Although subsequent improvements attempt to interleave ODE and SDE steps~\citep{li2025mixgrpo}, this paradigm fails to resolve the forward inconsistency between stochastic training trajectories and deterministic inference rollouts. 
Furthermore, by failing to implicitly distill the guidance signal into the vector field, these methods retain a structural reliance on CFG~\citep{ho2022classifier} to maintain generation quality, thereby imposing computational overhead during inference.

An alternative paradigm leverages implicit optimization and velocity field guidance to bypass explicit density estimation. 
Early reward-weighted methods~\citep{lee2023aligning} are typically restricted to offline settings and lack explicit penalization mechanisms for low-quality samples.
Recent online RL frameworks, such as DiffusionNFT~\citep{zheng2025diffusionnft}, have reformulated policy optimization as a supervised velocity regression task, achieving notable improvements in both training efficiency and generation performance over GRPO-style methods. 
As our analysis shows, its implicit positive and negative velocity fields induce two regression targets with fixed residual coefficients, while normalized rewards determine their relative loss weights. This leaves open how to derive sample-dependent target corrections directly from the reward-maximizing policy.

Herein, we propose Flow Advantage-Weighted Rectification (\textbf{FlowAWR}), an online RL framework that derives velocity regression targets from reward-maximizing policies without introducing stochastic sampling trajectories.
Our framework originates from the exact closed-form optimal policy derived from reward maximization under a KL penalty relative to the reference policy.
By mapping this optimal terminal distribution onto the intermediate probability paths, we analytically derive the corresponding optimal velocity field.
We demonstrate that this optimal field is equivalent to rectifying the reference velocity field by an advantage-weighted expectation.
Consequently, FlowAWR recasts flow RL as a supervised regression task toward this theoretical optimum. In practice, a group-based approximation to the derived advantage uses relative rewards to determine sample-dependent residual coefficients in the regression targets.

Our method preserves the architectural benefits of likelihood-free approaches by eliminating the forward inconsistency induced by SDE sampling, while its advantage-weighted updates implicitly distill alignment trajectories into the network parameters to enable CFG-free generation.
To evaluate the efficacy of FlowAWR, we conduct systematic experiments on the SD3.5-Medium (2.5B parameters) model~\citep{esser2024scaling}. 
Quantitative results indicate that, in the single-reward setting, FlowAWR achieves a $2\times$ to $5\times$ convergence acceleration alongside improved performance compared to DiffusionNFT. 
For instance, when optimizing for PickScore, FlowAWR elevates the metric to 24.12 within 1.2k steps. In contrast, DiffusionNFT requires 2.0k steps to reach 23.82, and FlowGRPO necessitates over 4k steps while relying on CFG to achieve 23.50. 
Furthermore, under the complex distribution shifts induced by multi-reward optimization, our sequential multi-reward training maintains optimization stability and mitigates out-of-domain aesthetic degradation.

Our primary contributions are summarized as follows:
\begin{itemize}[noitemsep, topsep=0pt]
\item We introduce \textbf{FlowAWR}, an SDE-free online RL framework based on velocity regression, and provide a \textbf{theoretical interpretation} of DiffusionNFT's implicit positive and negative velocity fields through their shared residual-rectification structure.
\item We characterize the \textbf{optimal velocity field} by propagating the closed-form optimal policy to intermediate optimal distributions along the probability path, ultimately instantiating this theoretical optimum via a \textbf{group-based approximated advantage function}.
\item We conduct quantitative evaluations across both single- and multi-objective optimization settings, indicating that FlowAWR accelerates convergence rates and improves empirical alignment performance compared to baseline methodologies.
\end{itemize}

\section{Background}

\subsection{Flow Matching and Rectified Flow}
\label{subsec:bg_fm}
For noise $x_0 \sim p_0$ and data $x_1 \sim p_1$, we adopt the linear path of Flow Matching and Rectified Flow~\citep{lipman2022flow, liu2023flow}:
\begin{equation*}
    x_t = t x_1 + (1 - t) x_0, \qquad t \in [0,1].
\end{equation*}
Each pair defines a constant interpolation velocity, $u_t(x_t|x_1) = x_1 - x_0 = (x_1 - x_t)/(1-t)$ for $t<1$. The model $v_\theta(x_t,t)$ is trained to regress this target:
\begin{equation}
    \mathcal{L}_{\text{FM}}(\theta) = \mathbb{E}_{t \sim \mathcal{U}[0,1], x_1 \sim p_1, x_0 \sim p_0} \left[ \| v_\theta(x_t, t) - u_t(x_t|x_1) \|^2 \right].
    \label{eq:fm_loss}
\end{equation}
At inference, the learned field generates samples by integrating the deterministic ODE $\dd x_t = v_\theta(x_t,t)\,\dd t$ from the noise prior. These model trajectories need not coincide with the linear interpolations used for training.

At the population optimum, velocity regression recovers the posterior average of the conditional fields~\citep{lipman2022flow, albergo2025stochastic}. Writing the posterior data mean as $\hat{x}_1(x_t) = \mathbb{E}[x_1 \mid x_t]$, the marginal velocity is
\begin{equation}
    v(x_t, t) = \mathbb{E}\left[u_t(x_t|x_1) \mid x_t\right] = \frac{\hat{x}_1(x_t) - x_t}{1-t}.
    \label{eq:posterior_mean}
\end{equation}

\begin{comment}
    
\subsection{Policy Gradient Algorithms for Flow Matching}
\label{subsec:bg_pg}

Conventional RL algorithms like PPO~\citep{schulman2017proximal} or GRPO~\citep{shao2024deepseekmath} rely on stochastic policies to estimate gradients. However, Flow Matching~\citep{lipman2022flow} generates samples via a deterministic ODE ($\dd x_t = v_\theta(x_t, t) \,\dd t$)~\citep{chen2018neural}, where the absence of randomness hinders the direct application of trajectory-level RL. Recent works~\citep{black2023training, fan2023dpok, liu2025flow, xue2025dancegrpo} formulate the generative sampling process as a multi-step MDP. 

Specifically, FlowGRPO~\citep{liu2025flow, xue2025dancegrpo} overcomes this constraint by reformulating the continuous flow via SDEs~\citep{song2019generative, song2020scorebased}. It augments the deterministic velocity field $v_\theta$ with a time-dependent diffusion term $g_t$:
\begin{equation*}
    \dd x_t = \left[ v_\theta(x_t, t) + \frac{g_t^2}{2t} \left( x_t + (1 - t)v_\theta(x_t, t) \right) \right] \dd t + g_t \,\dd w_t.
\end{equation*}
This formulation induces a tractable Gaussian transition kernel between discrete timesteps, which allows for the computation of trajectory likelihoods and enables the direct application of policy gradient algorithms such as GRPO.
\end{comment}

\subsection{Reinforcement Learning as Supervised Fine-tuning}
\label{subsec:bg_nft}

DiffusionNFT~\citep{zheng2025diffusionnft} treats RL as supervised velocity regression rather than trajectory-level policy optimization. It uses reward feedback to distinguish positive and negative components of the reference policy, represented implicitly by mixing the trainable field $v_\theta$ with the reference field $v^\old$. Omitting their shared arguments $(x_t,t)$, the mixed velocities are
\begin{equation}
    v_{\theta}^{\pm} := v^\old \pm \beta(v_\theta - v^\old),
    \label{eq:nft_coupled}
\end{equation}
where the fixed scale $\beta$ controls their opposite deviations from the reference. For a terminal sample $x_1$ with normalized reward $r \in [0,1]$, the two fields are regressed toward $u_t = u_t(x_t|x_1)$ through
\begin{equation}
    \mathcal{L}_{\text{NFT}}(\theta) = \mathbb{E} \left[ r \|v_{\theta}^+ - u_t\|^2 + (1 - r) \|v_{\theta}^- - u_t\|^2 \right].
    \label{eq:nft_loss}
\end{equation}
Higher rewards emphasize positive-field regression, while lower rewards emphasize the negative-field term. Both terms update the same trainable field $v_\theta$ through the coupled construction above, incorporating reward feedback without evaluating trajectory-level likelihoods.

\section{Method}

\begin{figure}[!t]
  \centering
  \includegraphics[width=0.98\linewidth]{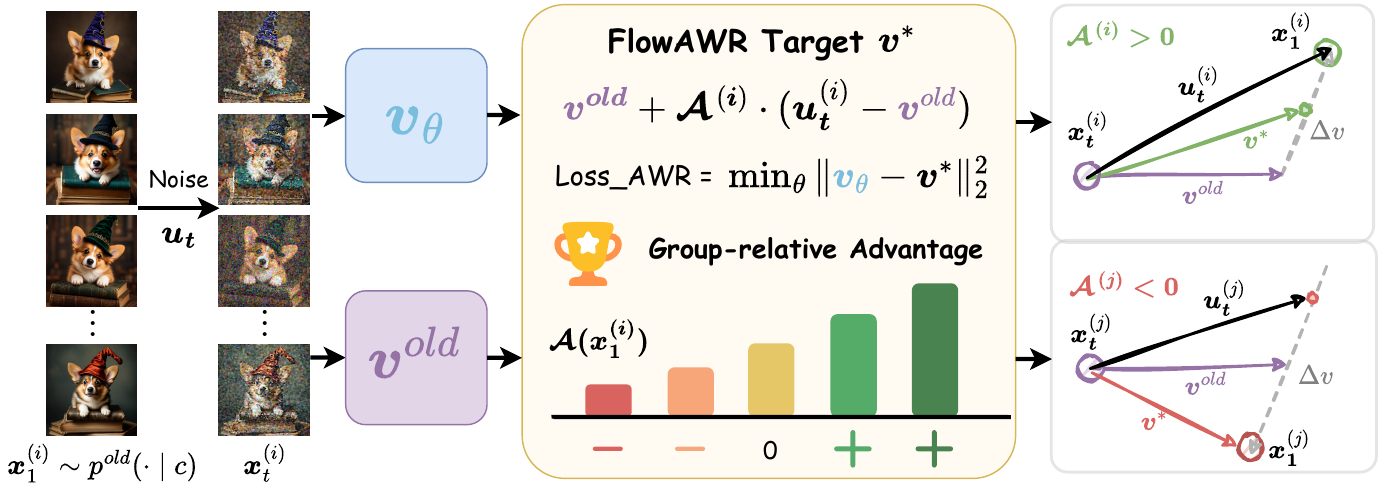}
    \vspace{-0.08in}
  \caption{\textbf{Overview of FlowAWR.} The model regresses $v_\theta$ toward $v^\old+\mathcal{A}(u_t-v^\old)$, using group-relative advantages to control the direction and magnitude of velocity rectification.}
   \vspace{-0.1in}
  \label{fig:pipeline}
\end{figure}

\subsection{Revisiting Optimal Policy from Fixed Experience Buffer}
\label{subsec:method_prelims}

At each online iteration, samples from the current policy $\pi^\old$ form an experience buffer $\mathcal{D}$. We hold the buffer and reference policy fixed during policy optimization and refresh them between iterations. For this fixed-reference step, we consider KL-regularized reward maximization~\citep{peters2010relative, schulman2015trust, levine2018reinforcement, peng2019advantage}:
\begin{equation}
    \max_{\pi_\theta} \mathbb{E}_{s \sim \mathcal{D}} \left[ \mathbb{E}_{a \sim \pi_\theta(\cdot|s)} R(s, a) - \gamma D_{\text{KL}}(\pi_\theta(\cdot|s) || \pi^\old(\cdot|s)) \right],
    \label{eq:constrained_objective}
\end{equation}
Here $s$ and $a$ denote the state and action, $R(s,a)$ the reward, and $\gamma>0$ the KL penalty weight.

Optimizing over normalized action distributions yields a closed-form policy~\citep{peters2010relative, rafailov2023direct}. With $Z(s)$ denoting the state-dependent partition function, the solution is
\begin{equation}
    \pi^*(a|s) = \frac{1}{Z(s)} \pi^\old(a|s) \exp\left(\frac{R(s, a)}{\gamma}\right).
    \label{eq:optimal_policy_sol}
\end{equation}
The derivation is provided in Appendix~\ref{subsec:derivation_optimal_policy}. This solution reweights the reference policy toward higher-reward actions, with $\gamma$ controlling the strength of the reward tilt. And for generative Flow Matching, the state $s$ corresponds to the prompt $c$ and the action $a$ to the generated sample $x_1$. The optimal policy therefore specifies a target terminal distribution.

\subsection{Derivation of Optimal Velocity Field}

Building upon the optimal policy in Eq.~\eqref{eq:optimal_policy_sol}, we first identify the optimal terminal distribution $p_1^*(x_1 | c)$, which represents the theoretical optimum for reward maximization on the current experience buffer, regardless of the exploration mechanism (e.g., SDE-based policy optimization):
\begin{equation*}
    p^*(x_1 | c) = \frac{1}{Z(c)} p^\old(x_1 | c) \exp\left(\frac{R(c, x_1)}{\gamma}\right),
\end{equation*}
where $c$ denotes a prompt, and $p^\old(x_1 | c)$ represents the reference policy. To derive the corresponding velocity field, the explicit form of the intermediate marginal distribution $p^*(x_t | c)$ is required:

\begin{proposition}[\textbf{Intermediate Marginals}]
\label{prop:optimal_posterior}
Given the optimal terminal distribution $p_1^*$, the corresponding marginal density at $t$ relates to the reference density $p_t^\old$ via a value function $\phi(x_t, c, t)$:
\begin{equation}
    p_t^*(x_t | c) = \frac{1}{Z(c)} p_t^\old(x_t | c) \cdot \phi(x_t, c, t),
    \label{eq:p_t^*}
\end{equation}
where $\phi(x_t, c, t) = \mathbb{E}_{x_1 \sim p^\old(\cdot | x_t, c)} \left[ \exp\left(\frac{R(c, x_1)}{\gamma}\right) \right]$. Proof in Appendix~\ref{app:proof_proposition_1}.
\end{proposition}

With the density relationship established, the explicit form of the optimal velocity field emerges. Combining the velocity form in Eq.~\eqref{eq:posterior_mean} with Tweedie's formula ~\citep{efron2011tweedie}, which links the posterior mean to the marginal score via $\hat{x}_1(x_t, c) = \frac{x_t}{t} + \frac{(1-t)^2}{t} \nabla_{x_t} \log p_t(x_t | c)$, we observe that the velocity is strictly determined by the score function of $p_t^*(x_t | c)$ in Eq.~\eqref{eq:p_t^*}:
\begin{align}
v^*(x_t, c, t) &= \frac{x_t}{t} + \frac{1-t}{t} \nabla_{x_t} \log p^*_t(x_t | c) \nonumber \\
            &= \frac{x_t}{t} + \frac{1-t}{t} \nabla_{x_t} \left[ \log p^\old_t (x_t | c) + \log \phi(x_t, c, t) \right] \nonumber \\
            &= v^\old(x_t, c, t) + \frac{1-t}{t} \nabla_{x_t} \log \phi(x_t, c, t).
\label{eq:value_gradient_velocity}
\end{align}
While Eq.~\eqref{eq:value_gradient_velocity} provides a theoretically rigorous direction for rectification, directly computing $\nabla_{x_t} \log \phi(x_t, c, t)$ is intractable. Herein, we expand the gradient into an expectation over samples, which reveals a connection to advantage-weighted regression~\citep{peng2019advantage, zhang2025energy}.

\begin{theorem}[\textbf{Advantage-Weighted Rectification}]
\label{thm:advantage_rectification}
The optimal velocity field $v^*(x_t, c)$ derived in Eq.~\eqref{eq:value_gradient_velocity} can be equivalently expressed as the reference field rectified by an advantage-weighted residual expectation (proof in Appendix~\ref{app:proof_theorem_1}):
\begin{equation*}
    v^*(x_t, c, t) = v^\old(x_t, c, t) + \mathbb{E}_{x_1 \sim p^\old(\cdot | x_t, c)} \left[ \mathcal{A}(x_1, x_t) \left( u_t(x_t|x_1) - v^\old(x_t, c, t) \right) \right],
    \label{eq:theorem_final}
\end{equation*}
where $\mathcal{A}(x_1, x_t)$ is the Centered Advantage, constructed by leveraging the zero-expectation property $\mathbb{E}[u_t - v^\old] = 0$ to subtract a constant baseline of $1$:
\begin{equation}
    \mathcal{A}(x_1, x_t) = \frac{\exp\left(\frac{R(c, x_1)}{\gamma}\right)}{\phi(x_t, c, t)} - 1.
    \label{eq:advantage_base}
\end{equation}
\end{theorem}
This formulation reveals that the optimal field $v^*$ is obtained by dynamically rectifying the reference field $v^\old$ with the velocity residual $(u_t - v^\old)$ based on $\mathcal{A}$.
This exhibits an explicit awareness of sample quality, enabling the formulation to quantitatively differentiate between positive and negative samples through the polarity of $\mathcal{A}$, thereby reinforcing or repelling trajectories accordingly.

\subsection{Flow Reinforcement via Advantage-Weighted Rectification}
\label{subsec:flow_reinforcement}

Based on Theorem~\ref{thm:advantage_rectification}, the policy optimization process can be cast as a supervised regression task towards the optimal velocity field. Ideally, the formulation is the minimization of the distance between $v_\theta$ and the theoretical optimum $v^*$, as defined in the standard objective in Eq.~\eqref{eq:fm_loss}:
\begin{equation*}
    \mathcal{L}_{\mathrm{Ideal}}(\theta) = \mathbb{E}_{t, x_t \sim p_t^\old} \left[ \left\| v_\theta(x_t, c, t) - v^*(x_t, c, t) \right\|^2 \right].
\end{equation*}

However, directly optimizing $\mathcal{L}_{\mathrm{Ideal}}$ is intractable due to the expectation over $p^\old(x_1 | x_t, c)$ embedded within $v^*(x_t, c, t)$. To circumvent this, we propose the Advantage-Weighted Rectification (AWR) Loss, which learns from individual samples by constructing a stochastic rectified target:
\begin{equation}
\label{eq:awr_loss}
\begin{split}
    &\mathcal{L}_{\mathrm{AWR}}(\theta) = \mathbb{E}_{t \sim \mathcal{U}[0,1], x_1 \sim p^\old(\cdot|c), x_0 \sim p_0(\cdot)} \bigg[ \\
    &\quad \Big\| v_\theta(x_t, c, t) - \Big( v^\old(x_t, c, t) + \mathcal{A}(x_1, x_t) \cdot \big( u_t(x_t | x_1) - v^\old(x_t, c, t) \big) \Big) \Big\|^2 \bigg].
\end{split}
\end{equation}

\begin{theorem}[\textbf{Optimization Equivalence}]
\label{thm:optimization_equivalence}
The Advantage-Weighted Rectification Loss $\mathcal{L}_{\mathrm{AWR}}(\theta)$ and the Ideal Loss $\mathcal{L}_{\mathrm{Ideal}}(\theta)$ share the same global optimum. Specifically, the gradients of both objectives with respect to $\theta$ are identical:
\begin{equation*}
    \nabla_\theta \mathcal{L}_{\mathrm{AWR}}(\theta) = \nabla_\theta \mathcal{L}_{\mathrm{Ideal}}(\theta).
\end{equation*}
which indicates that minimizing $\mathcal{L}_{\mathrm{AWR}}$ drives the parameterized velocity field $v_\theta$ to converge to the theoretical optimal field $v^*(x_t, c, t)$. Proof in Appendix~\ref{app:proof_theorem_2}.
\end{theorem}

Theorem~\ref{thm:optimization_equivalence} establishes a novel perspective for Flow RL: directly learning the optimal terminal distribution derived from each experience buffer. This results in a new off-policy paradigm that casts policy optimization as supervised regression on advantage-weighted samples (Figure~\ref{fig:pipeline}). 

In practice, we use \textbf{Group-Based Estimation} with $G$ images $\{x_1^{(i)}\}_{i=1}^G$ sampled from the reference policy for each prompt $c$, following group-based methods~\citep{liu2025flow, zheng2025diffusionnft}. We approximate the $\phi(x_t,c,t)$ by the within-prompt mean of exponentiated rewards, yielding:
\begin{equation}
\mathcal{A}(x_1^{(i)}) \approx \frac{\exp\left(\frac{R(c, x_1^{(i)})}{\gamma}\right)}{\frac{1}{G} \sum_{j=1}^G \exp\left(\frac{R(c, x_1^{(j)})}{\gamma}\right)} - 1. 
\label{eq:Softmax_Advantage}
\end{equation}
Substituting these advantages into the AWR loss gives the group-averaged objective, with the shared conditioning $c$ omitted from the velocity fields:
\begin{equation*}
\label{eq:group_loss}
\begin{split}
&\mathcal{L}_{\mathrm{AWR}}^{(c)}(\theta)=\frac{1}{G}\sum_{i=1}^{G}\mathbb{E}_{t,x_0^{(i)}\sim p_0(\cdot), x_1^{(i)} \sim p_1^{\mathrm{old}}(\cdot | c)}\bigg[ \\
&\quad \Big\| v_\theta(x_t^{(i)},t)-\Big(v^\old(x_t^{(i)},t)+\mathcal{A}(x_1^{(i)})\cdot\big(u_t(x_t^{(i)}|x_1^{(i)})-v^\old(x_t^{(i)},t)\big)\Big)\Big\|^2 \bigg].
\end{split}
\end{equation*}

We apply \textbf{Adaptive Scaling} by setting $\gamma$ to the standard deviation of rewards across the current experience buffer, adapting the KL penalty to the reward scale. For \textbf{Soft Online Evolution}, the reference parameters follow the EMA update $\theta^\old \leftarrow \mu_i\theta^\old+(1-\mu_i)\theta$, where $\mu_i$ is the decay rate at iteration $i$; its schedule is examined in Section~\ref{subsec:ablation_study}. In our SD3.5-Medium experiments, we also use \textbf{CFG-Free Generation} during rollouts and inference, as in DiffusionNFT~\citep{zheng2025diffusionnft}.

\subsection{Unified Perspective on Flow Reinforcement}
\label{subsec:discussion}

Our FlowAWR framework not only provides a foundation for supervised flow policy optimization but also establishes a unifying lens through which both standard supervised fine-tuning and recent heuristic methods can be understood. Together, Theorems~\ref{thm:advantage_rectification} and~\ref{thm:optimization_equivalence} yield the following sample-wise regression target for learning the optimal velocity field:
\begin{equation}
    v^*(x_t, c, t) = v^{\text{old}}(x_t, c, t) + \mathcal{A}(x_1, x_t) \cdot \left( u_t(x_t|x_1) - v^{\text{old}}(x_t, c, t) \right).
    \label{eq:optimal_velocity_explicit}
\end{equation}

\begin{figure}[t!]
  \centering
  \begin{minipage}{\linewidth}
  \captionsetup{type=table}
  \centering
  \caption{Our evaluation results compared with DiffusionNFT. \protect\colorbox{baselinegray}{Gray-colored}: In-domain target task. The DiffusionNFT OCR row uses its official task-specific configuration ($\beta=0.1$, delayed old-policy EMA); Appendix~\ref{subsec:nft_ocr_configuration} compares it with the general configuration.}
  \vspace{-2pt}
  \label{tab:baseline_comparison_grouped}
  \resizebox{\textwidth}{!}{%
    \setlength{\tabcolsep}{4pt}
    \begin{tabular}{lccccccccc}
      \toprule
      \multirow{2}{*}{\textbf{Method}} & \multirow{2}{*}{\textbf{\#Iter}} & \multicolumn{3}{c}{\textbf{Task Metric}} & \multicolumn{5}{c}{\textbf{DrawBench Metric}} \\
      \cmidrule(lr){3-5} \cmidrule(lr){6-10}
      & & \textbf{GenEval} & \textbf{OCR} & \textbf{PickScore} & \textbf{ClipScore} & \textbf{HPSv2.1} & \textbf{Aesthetic} & \textbf{ImgRwd} & \textbf{PickScore} \\
      \midrule
      SD3.5-M + CFG & --- & 0.63 & 0.59 & 21.75 & 0.285 & 0.279 & 5.36 & 0.85 & 22.34 \\
      \midrule
      DiffusionNFT & 1k & \cellcolor{baselinegray}\textbf{0.98} & 0.35 & 20.19 & 0.274 & 0.240 & 4.99 & 0.40 & 21.78 \\
      FlowAWR & \textbf{0.5k} & \cellcolor{baselinegray}\textbf{0.98} & 0.37 & 20.25 & 0.274 & 0.252 & 5.09 & 0.43 & 21.92 \\
      \midrule
      DiffusionNFT & 0.4k & 0.20 & \cellcolor{baselinegray}\textbf{0.98} & 19.15 & 0.233 & 0.168 & 4.79 & -1.27 & 20.15 \\
      FlowAWR & \textbf{0.26k} & 0.29 & \cellcolor{baselinegray}0.975 & 19.43 & 0.246 & 0.193 & 5.00 & -0.78 & 20.51 \\
      \midrule
      DiffusionNFT & 2.0k & 0.55 & 0.67 & \cellcolor{baselinegray}23.82 & 0.272 & 0.316 & 6.19 & 1.31 & 24.09 \\
      FlowAWR & \textbf{1.2k} & 0.61 & 0.72 & \cellcolor{baselinegray}\textbf{24.12} & 0.275 & 0.315 & 6.28 & 1.36 & 24.29 \\
      \bottomrule
    \end{tabular}%
  }
  \end{minipage}\par
  \vspace{-0.018in}
  \begin{minipage}{\linewidth}
  \captionsetup{type=figure}
    \centering
    \includegraphics[width=\linewidth]{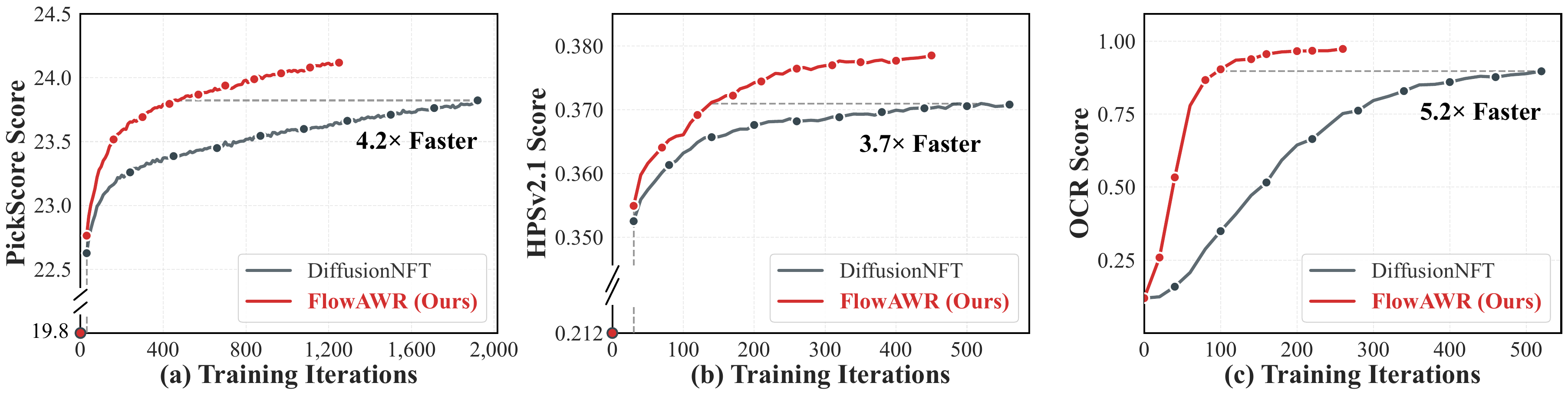}
    \vspace{-0.2in}
    \caption{Head-to-head comparison between FlowAWR and DiffusionNFT on single-reward optimization tasks: (a) PickScore, (b) HPSv2.1, and (c) OCR.}
    \label{fig:main_efficiency}
    \vspace{-0.15in}
  \end{minipage}
\end{figure}

\paragraph{Theoretical Interpretation of DiffusionNFT.}
DiffusionNFT's implicit positive and negative velocity fields induce regression targets that share the residual structure above. For a fixed $\beta>0$, factoring out $\beta$ from its coupled regression loss in Eqs.~\eqref{eq:nft_coupled} and~\eqref{eq:nft_loss} gives
\begin{equation*}
\begin{split}
    \mathcal{L}_{\text{NFT}} \propto \mathbb{E} \bigg[ r &\left\| v_\theta - \left(v^\old + \frac{1}{\beta}(u_t - v^\old)\right) \right\|^2
    + (1 - r) \left\| v_\theta - \left(v^\old - \frac{1}{\beta}(u_t - v^\old)\right) \right\|^2 \bigg].
\end{split}
\end{equation*}
This factorization provides a theoretical interpretation of DiffusionNFT's implicit positive and negative velocity fields: their regression targets take the advantage-weighted rectification form in Eq.~\eqref{eq:optimal_velocity_explicit}, with fixed coefficients $\mathcal{A}_{\mathrm{NFT}}=\pm1/\beta$.

FlowAWR instead derives its rectification coefficient from the reward-tilted optimal policy. The theoretical coefficient normalizes the exponentiated reward by the conditional value function $\phi$; in practice, we approximate this normalization using the within-prompt mean of exponentiated rewards. Combined with reward-adaptive $\gamma$, this construction determines sample-dependent rectification magnitudes without the separate mixing parameter $\beta$. The ablations in Section~\ref{subsec:ablation_study} support reward-adaptive scaling and show reward degradation in the fixed-coefficient single-branch variants.

\paragraph{Reduction to Supervised Fine-Tuning.}
Suppose all samples drawn from $p_1^{\mathrm{old}}(\cdot\mid c)$ are treated as equally high-quality supervision and assigned the same reward. The normalized exponential reward then equals one. Without subtracting the unit baseline, the regression objective reduces to
\begin{equation*}
\mathcal{L}(\theta)
=
\mathbb{E}\left[
\left\|v_\theta-\left(v^\old+1\cdot(u_t-v^\old)\right)\right\|^2
\right]
=
\mathbb{E}\left[\left\|v_\theta-u_t\right\|^2\right].
\end{equation*}

\section{Experiments}
\label{sec:experiments}

\subsection{Experimental Setup}
\label{subsec:experimental_setup}

Our main experiments are based on SD3.5-Medium~\citep{esser2024scaling} at $512 \times 512$ resolution, with most experimental settings following FlowGRPO~\citep{liu2025flow} and DiffusionNFT~\citep{zheng2025diffusionnft}. Additional validation on FLUX.1-dev~\citep{labs2025flux} is provided in Appendix~\ref{subsec:flux_experiments}.

\textbf{Tasks and Evaluation Metrics.} We categorize the text-to-image optimization tasks into two paradigms: (1) \textit{Rule-based constraints}, utilizing GenEval~\citep{ghosh2023geneval} for compositional spatial reasoning and OCR accuracy for explicit text rendering. (2) \textit{Model-based human preference}, employing PickScore~\citep{kirstain2023pick} as the primary reward signal. To assess out-of-domain generalization, models are additionally evaluated on the DrawBench~\citep{saharia2022photorealistic} benchmark using an array of evaluation metrics: PickScore, CLIPScore~\citep{hessel2021clipscore}, HPSv2.1~\citep{wu2023human}, Aesthetic Score~\citep{schuhmann2022laion}, and ImageReward~\citep{xu2023imagereward}. 

\textbf{Training and Evaluation.} We finetune with LoRA ($\alpha = 64, r = 32$). During training, each epoch processes 48 prompts, with an intra-group sample size of $G = 24$. For trajectory rollouts, we use $T = 10$ sampling steps for both head-to-head baseline comparisons, ablation studies and multi-reward joint training. 
All final evaluations are performed using a 40-step first-order ODE solver. Except in the EMA ablation, FlowAWR uses the same reference-policy EMA schedule, $\mu_i=\min(0.001i,0.5)$, across all tasks. Further details and settings are provided in Appendix~\ref{sec:experiment_details}.

\subsection{Head-to-Head Comparison on Single Rewards}
\label{subsec:head_to_head}
\begin{table}[t!]
  \centering
  \caption{Evaluations for the multi-reward setting. \colorbox{black!15}{Gray-colored}: In-domain reward. $^\ddagger$Evaluated under $1024 \times 1024$ resolution. \textbf{Bold}: best result. \underline{Underline}: second best.}
  \label{tab:multi_reward}
  \vspace{-0.1in}
  \resizebox{\textwidth}{!}{%
    \setlength{\tabcolsep}{4pt}
    \begin{tabular}{lccccccccc}
      \toprule
      \multirow{2}{*}{\textbf{Model}} & \multirow{2}{*}{\textbf{\#Iter}} & \multicolumn{3}{c}{\textbf{Task Metric}} & \multicolumn{5}{c}{\textbf{DrawBench Metric}} \\
      \cmidrule(lr){3-5} \cmidrule(lr){6-10}
      & & \textbf{GenEval} & \textbf{OCR} & \textbf{PickScore} & \textbf{ClipScore} & \textbf{HPSv2.1} & \textbf{Aesthetic} & \textbf{ImgRwd} & \textbf{PickScore} \\
      \midrule
      
      SD-XL$^{\ddagger}$ & --- & 0.55 & 0.14 & 21.74 & 0.287 & 0.280 & 5.60 & 0.76 & 22.42 \\
      SD3.5-L$^{\ddagger}$ & --- & 0.71 & 0.68 & 22.24 & 0.289 & 0.288 & 5.50 & 0.96 & 22.91 \\
      FLUX.1-dev & --- & 0.66 & 0.59 & 22.55 & \underline{0.295} & 0.274 & 5.71 & 0.96 & 22.84 \\
      \midrule
      
      SD3.5-M (w/o CFG) & --- & 0.24 & 0.12 & 19.90 & 0.237 & 0.204 & 5.13 & -0.58 & 20.51 \\
      + CFG & --- & 0.63 & 0.59 & 21.75 & 0.285 & 0.279 & 5.36 & 0.85 & 22.34 \\
      \quad + FlowGRPO & $>$5k & \cellcolor{baselinegray}\textbf{0.95} & 0.66 & 21.67 & 0.293 & 0.274 & 5.32 & 1.06 & 22.37 \\
      & 2k & 0.66 & \cellcolor{baselinegray}\underline{0.92} & 21.80 & 0.290 & 0.280 & 5.32 & 0.95 & 22.44 \\
      & 4k & 0.54 & 0.68 & \cellcolor{baselinegray}23.33 & 0.280 & 0.316 & 5.90 & 1.29 & 23.50 \\
      \midrule
      
      + DiffusionNFT & 1.7k & \cellcolor{baselinegray}0.93 & \cellcolor{baselinegray}0.90 & \cellcolor{baselinegray}\underline{23.36} & \cellcolor{baselinegray}0.293 & \cellcolor{baselinegray}\underline{0.330} & \underline{5.99} & \underline{1.48} & \underline{23.79} \\
      + FlowAWR                  & 1.5k & \cellcolor{baselinegray}\underline{0.93} & \cellcolor{baselinegray}\underline{0.91} & \cellcolor{baselinegray}\textbf{23.60} & \cellcolor{baselinegray}\textbf{0.296} & \cellcolor{baselinegray}\textbf{0.341} & \textbf{6.04} & \textbf{1.51} & \textbf{23.91} \\
      \bottomrule
    \end{tabular}%
  }\par
  \vspace{-0.1in}
\end{table}

\begin{figure}[t!]
  \centering
  \includegraphics[width=\textwidth]{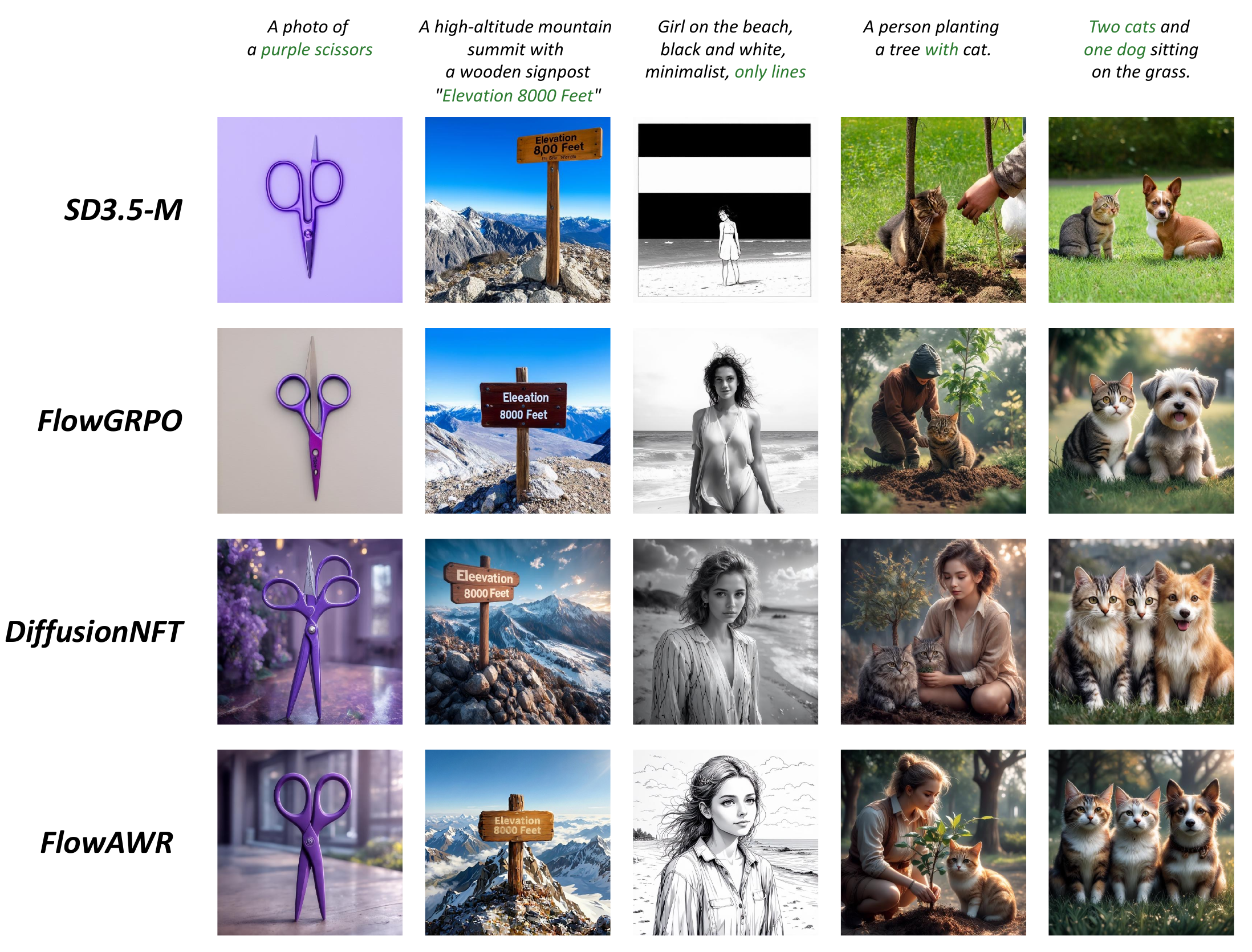}\par
  \vspace{-0.1in}
  \caption{Qualitative Comparison. The text prompts are sampled from GenEval, OCR, PickScore, and DrawBench evaluations, illustrating the generation fidelity of our FlowAWR.}
  \label{fig:qualitative_comp}
  \vspace{-0.2in}

\end{figure}

Table~\ref{tab:baseline_comparison_grouped} and Figure~\ref{fig:main_efficiency} compare FlowAWR with DiffusionNFT on single-reward optimization. FlowAWR achieves comparable or higher target rewards with fewer updates on GenEval and PickScore. For OCR, the table uses DiffusionNFT's task-specific configuration, while the learning curves use its general configuration; Appendix~\ref{subsec:nft_ocr_configuration} provides the comparison. GPU-hour comparisons, including AWM and TempFlow-GRPO, are reported in Appendix~\ref{subsec:training_cost}.

As evidenced in Table~\ref{tab:baseline_comparison_grouped}, FlowAWR exhibits reduced reward hacking during out-of-domain evaluations. However, exclusively optimizing a single rule-based objective (e.g., GenEval or OCR) typically degrades overall image aesthetics due to the absence of holistic visual constraints. To address this, we explore multi-reward sequential training in the subsequent section.

\subsection{Multi-Reward Sequential Training}

To mitigate the aesthetic degradation associated with isolated rule-based objectives, we extend FlowAWR to a multi-reward setting via a sequential optimization strategy. Using the CFG-free SD3.5-Medium, we first optimize the policy using a composite reward signal comprising PickScore, CLIPScore, and HPSv2.1 on the Pick-a-Pic dataset~\citep{kirstain2023pick}. 
We then continue training through stages that combine these preference rewards with GenEval or OCR, resuming the full training state at each transition (Appendix~\ref{subsec:appendix_experimental_setup}). This sequential approach improves adherence to structural rules while preserving visual fidelity, as shown in Table~\ref{tab:multi_reward} and Figure~\ref{fig:qualitative_comp}.
\subsection{Ablation Study}
\label{subsec:ablation_study}

\textbf{Group Size.} Figure~\ref{fig:ablation_study}(a,b) compares $G\in\{8,24,32\}$ on OCR and PickScore, with larger groups achieving higher final rewards. We adopt $G=24$ to balance reward and sampling cost.
\textbf{Adaptive $\gamma$ Scaling.} The reward-adaptive choice achieves a higher final OCR reward than the fixed values $\gamma\in\{0.001,0.01\}$ (Figure~\ref{fig:ablation_study}(c)), supporting adaptation to the reward scale.
\textbf{Adaptive Advantage.} We compare adaptive advantages with fixed coefficients $\mathcal{A}\in\{\pm1,\pm10\}$, using one NFT-style regression target per sample and reward-dependent weights outside the squared loss. Both fixed settings suffer reward degradation on OCR and PickScore, whereas adaptive advantages yield increasing rewards (Figure~\ref{fig:ablation_study}(d,e)).
\textbf{Soft Online Update.} Our default schedule $\mu_i=\min(0.001i,0.5)$ achieves a higher OCR reward than the faster schedule $\mu_i=\min(0.01i,0.9)$ (Figure~\ref{fig:ablation_study}(f)). On-policy updates ($\mu_i=0$) improve reward initially but collapses.

\begin{figure}[t!]
    \centering
    \includegraphics[width=1.0\linewidth]{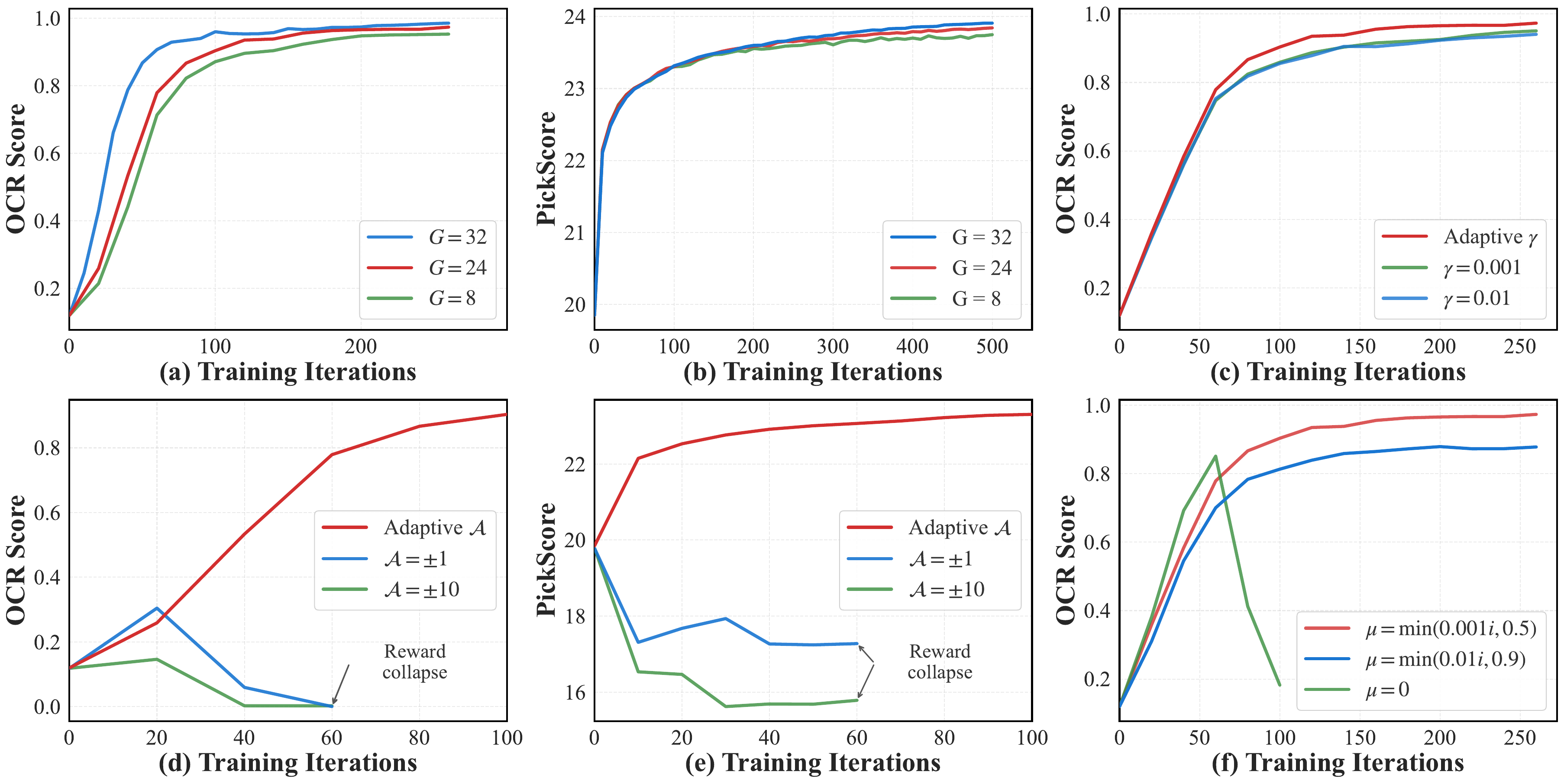}
    \vspace{-0.2in}
    \caption{Ablation studies of FlowAWR. (a,b) Group size $G$ on OCR and PickScore, respectively. (c) Adaptive versus fixed $\gamma$ on OCR. (d,e) Adaptive versus fixed-magnitude advantages on OCR and PickScore, respectively. (f) Reference-policy update schedules $\mu_i$ on OCR.}
    \label{fig:ablation_study}
    \vspace{-0.2in}
\end{figure}

\section{Related Work}
\label{sec:related_work}

%The broad application of RL in aligning LLMs~\cite{ouyang2022training, rafailov2023direct} has motivated efforts to align continuous generative models. Existing literature predominantly bifurcates into two paradigms.

\paragraph{Implicit Optimization and Guidance-based Methods.}
Reward-based adaptation includes weighted regression, differentiable reward optimization, and inference-time guidance.
Early RWR-style methods~\citep{lee2023aligning} favor high-reward samples through weighted or filtered fine-tuning.
Differentiating rewards through sampling~\citep{xu2023imagereward, prabhudesai2023aligning, clark2023directly} requires differentiable feedback and incurs backpropagation costs, while inference-time guidance~\citep{janner2022planning, jin2025inference} adds computation during sampling without additional parameter updates at inference time.
Recent methods connect reward optimization to pretraining objectives.
Advantage Weighted Matching (AWM)~\citep{xue2025advantage} formulates policy-gradient optimization through advantage-weighted score or flow matching.
DiffusionNFT~\citep{zheng2025diffusionnft} performs likelihood-free optimization through reward-weighted positive and negative velocity regression.

\paragraph{Density-Approximated Policy Gradient Methods.}
Trajectory-based methods formulate sampling as an MDP to apply standard RL algorithms~\citep{black2023training, fan2023dpok}.
Diffusion-DPO~\citep{wallace2024diffusion} and D3PO~\citep{yang2024using} use likelihood approximations for preference optimization, while online methods such as DDPO~\citep{black2023training}, FlowGRPO~\citep{liu2025flow}, and DanceGRPO~\citep{xue2025dancegrpo} optimize policies through stochastic transition likelihoods.
For flow models, SDE sampling provides tractable transition densities but introduces differences between training trajectories and deterministic ODE inference.
MixGRPO~\citep{li2025mixgrpo} interleaves ODE and SDE steps to reduce training cost.
TempFlow-GRPO~\citep{he2026tempflow} further branches trajectories at selected timesteps for localized credit assignment and applies noise-aware weighting to correct timestep-dependent gradient scaling.
These refinements improve trajectory-based optimization while retaining stochastic transitions during training.

\section{Conclusion}
\label{sec:conclusion}
We introduce \textit{Flow Advantage-Weighted Rectification} (\textbf{FlowAWR}), an off-policy online RL framework for continuous generative models. Starting from a KL-constrained reward maximization problem, FlowAWR derives an advantage-weighted velocity regression objective that recovers supervised fine-tuning under equal rewards without baseline subtraction. This analysis also provides a theoretical interpretation of DiffusionNFT's implicit positive and negative velocity fields, while motivating reward-derived, sample-dependent corrections to the regression targets.
Empirically, FlowAWR achieves $2\times$ to $5\times$ faster convergence than DiffusionNFT in single-reward tasks, while remaining robust under the distribution shifts induced by multi-reward alignment.

% =====================================================
% End of migrated body.
% =====================================================
\clearpage
\subsection*{AI use statement}
We used generative AI tools to assist with manuscript preparation, primarily for grammar correction, refinement of wording, and LaTeX formatting. AI-assisted checks do not replace author verification. The authors take responsibility for the final manuscript, including its technical claims and derivations.

\subsection*{Ethics statement}
Our experiments use publicly available pretrained models and established reward functions to evaluate FlowAWR. As a method for reward-based optimization of text-to-image models, FlowAWR may inherit biases from the underlying models and reward functions, and improvements in reward scores do not by themselves ensure safe or unbiased outputs. We have not identified additional ethical concerns specific to this study beyond these general risks.

\subsection*{Reproducibility statement}
We provide the derivation of the optimal policy in Appendix~\ref{subsec:derivation_optimal_policy} and proofs of the main theoretical results in Appendices~\ref{app:proof_proposition_1}, \ref{app:proof_theorem_1}, and~\ref{app:proof_theorem_2}. The experimental settings, advantage scaling, single-branch ablation configurations, and sequential multi-reward training procedure are described in Appendix~\ref{subsec:appendix_experimental_setup}. Additional baseline comparisons and model-specific settings are reported in Appendix~\ref{sec:experiment_details}, including the DiffusionNFT OCR configurations in Appendix~\ref{subsec:nft_ocr_configuration}.

\bibliography{iclr2026_conference}
\bibliographystyle{iclr2026_conference}

\clearpage
\appendix
\etocdepthtag.toc{appendix}
\etocsettagdepth{main}{none}
\etocsettagdepth{appendix}{subsection}
\etocsettocstyle{\section*{Appendix Contents}}{}
\hypersetup{linktoc=all}
\tableofcontents
\clearpage
\setcounter{table}{0}
\renewcommand{\thetable}{S\arabic{table}}

\section{Extended Theoretical Derivations and Proofs}

\subsection{Derivation of the Optimal Policy}
\label{subsec:derivation_optimal_policy}

This section provides the analytical derivation of the closed-form optimal policy $\pi^*$, which is strictly optimized over the current fixed experience buffer $\mathcal{D}$ subject to a KL divergence constraint against the reference policy.

While the iterative update of $\pi^{\text{old}}$ enables online RL, the step-wise optimization of $\pi_\theta$ fundamentally solves a constrained reward maximization problem over the empirical states in $\mathcal{D}$:
\begin{equation*}
J(\pi_\theta) = \mathbb{E}_{s \sim \mathcal{D}} \left[ \mathbb{E}_{a \sim \pi_\theta(\cdot|s)}[R(s, a)] - \gamma D_{\text{KL}}(\pi_\theta(\cdot|s) || \pi^{\text{old}}(\cdot|s)) \right],
\end{equation*}
where maximizing the global objective over $\mathcal{D}$ is mathematically equivalent to solving the optimization for each state $s \in \mathcal{D}$ independently. Expanding the terms for a specific state $s$ yields:
\begin{equation*}
J(\pi_\theta(\cdot|s)) = \int \pi_\theta(a|s) R(s, a) da - \gamma \int \pi_\theta(a|s) \log \frac{\pi_\theta(a|s)}{\pi^{\text{old}}(a|s)} da.
\end{equation*}

To ensure $\pi_\theta(\cdot|s)$ is a valid probability density, a state-dependent Lagrangian multiplier $\lambda(s)$ is introduced for the normalization constraint $\int \pi_\theta(a|s) da = 1$, yielding the Lagrangian functional:
\begin{equation*}
\begin{split}
L(\pi_\theta, \lambda) =& \int \pi_\theta(a|s) R(s, a) da - \gamma \int \pi_\theta(a|s) \log \frac{\pi_\theta(a|s)}{\pi^{\text{old}}(a|s)} da \\
&+ \lambda(s) \left( 1 - \int \pi_\theta(a|s) da \right).
\end{split}
\end{equation*}

Taking the functional derivative of $L$ with respect to the density $\pi_\theta(a|s)$ and setting it to zero yields the first-order optimality condition:
\begin{equation*}
\frac{\delta L}{\delta \pi_\theta(a|s)} = R(s, a) - \gamma \left( \log \frac{\pi_\theta(a|s)}{\pi^{\text{old}}(a|s)} + 1 \right) - \lambda(s) = 0.
\end{equation*}

Rearranging the terms to isolate $\pi_\theta(a|s)$:
\begin{align*}
\log \frac{\pi_\theta(a|s)}{\pi^{\text{old}}(a|s)} &= \frac{R(s, a)}{\gamma} - 1 - \frac{\lambda(s)}{\gamma}, \\
\pi_\theta(a|s) &= \pi^{\text{old}}(a|s) \exp\left( \frac{R(s, a)}{\gamma} \right) \exp\left( -1 - \frac{\lambda(s)}{\gamma} \right).
\end{align*}

Since the term $\exp\left( -1 - \frac{\lambda(s)}{\gamma} \right)$ is independent of $a$, it acts as a normalization constant. Defining its inverse as the partition function $Z(s)$:
\begin{equation*}
Z(s) = \exp\left( 1 + \frac{\lambda(s)}{\gamma} \right) = \int \pi^{\text{old}}(a|s) \exp\left( \frac{R(s, a)}{\gamma} \right) da.
\end{equation*}

Substituting $Z(s)$ back yields the theoretical optimal policy $\pi^*$ corresponding to Eq.~\eqref{eq:optimal_policy_sol}:
\begin{equation}
\pi^*(a|s) = \frac{1}{Z(s)} \pi^{\text{old}}(a|s) \exp\left(\frac{R(s, a)}{\gamma}\right).
\label{eq:appendix_optimal_policy_final}
\end{equation}

\subsection{Proof of Proposition 1}
\label{app:proof_proposition_1}

\textbf{Proposition 1.} \textit{The optimal marginal density at time $t$ relates to the reference density $p_t^{\text{old}}$ via a time-dependent value function $\phi(x_t, c, t)$:}
\begin{equation*}
    p_t^*(x_t | c) = \frac{1}{Z(c)} p_t^{\text{old}}(x_t | c) \cdot \phi(x_t, c, t),
\end{equation*}
\textit{where $\phi(x_t, c, t) = \mathbb{E}_{x_1 \sim p^{\text{old}}(\cdot | x_t, c)} \left[ \exp\left(\frac{R(c, x_1)}{\gamma}\right) \right]$.}

\begin{proof}
The marginal density $p_t^*(x_t|c)$ is obtained by marginalizing the joint distribution over the terminal state $x_1$:
\begin{equation*}
    p_t^*(x_t | c) = \int p(x_t | x_1) p_1^*(x_1 | c) dx_1.
\end{equation*}

The forward transition kernel $p(x_t | x_1)$ is determined by the ReFlow schedule (e.g., $x_t = t x_1 + (1-t) x_0$) and remains strictly independent of the policy. Substituting the optimal terminal distribution $p_1^*(x_1 | c) = \frac{1}{Z(c)} p_1^{\text{old}}(x_1 | c) \exp\left(\frac{R(c, x_1)}{\gamma}\right)$:
\begin{equation*}
    p_t^*(x_t | c) = \frac{1}{Z(c)} \int p(x_t | x_1) p_1^{\text{old}}(x_1 | c) \exp\left(\frac{R(c, x_1)}{\gamma}\right) dx_1.
\end{equation*}

Applying Bayes' theorem to the reference policy terms, the joint probability can be factorized as $p(x_t | x_1) p_1^{\text{old}}(x_1 | c) = p^{\text{old}}(x_1 | x_t, c) p_t^{\text{old}}(x_t | c)$. Substituting this factorization back into the integral formulation:
\begin{equation*}
    p_t^*(x_t | c) = \frac{1}{Z(c)} \int p_t^{\text{old}}(x_t | c) p^{\text{old}}(x_1 | x_t, c) \exp\left(\frac{R(c, x_1)}{\gamma}\right) dx_1.
\end{equation*}

Since the intermediate reference marginal $p_t^{\text{old}}(x_t | c)$ is independent of the integration variable $x_1$, it can be factored out of the integral:
\begin{equation*}
    p_t^*(x_t | c) = \frac{1}{Z(c)} p_t^{\text{old}}(x_t | c) \underbrace{\int p^{\text{old}}(x_1 | x_t, c) \exp\left(\frac{R(c, x_1)}{\gamma}\right) dx_1}_{\phi(x_t, c, t)}.
\end{equation*}

The integral term mathematically corresponds to the conditional expectation over the posterior distribution $p^{\text{old}}(x_1 | x_t, c)$, which explicitly defines the value function $\phi(x_t, c, t)$.
\end{proof}

\subsection{Derivation of Tweedie's Formula in Flow Matching}
\label{subsec:derivation_tweedie}

To establish the mathematical foundation, this section provides the analytical derivation of Tweedie's formula within the context of Flow Matching, which explicitly maps the posterior mean estimation to the marginal score function.

Under the ReFlow formulation, the forward probability path is constructed via a linear interpolation between the noise $x_0 \sim \mathcal{N}(0, I)$ and the data $x_1 \sim p_1$:
\begin{equation*}
    x_t = t x_1 + (1-t) x_0.
\end{equation*}

Given the Gaussian conditional distribution $p(x_t | x_1) = \mathcal{N}(x_t; t x_1, (1-t)^2 I)$, the corresponding conditional score function evaluates to:
\begin{equation*}
    \nabla_{x_t} \log p(x_t | x_1) = -\frac{x_t - t x_1}{(1-t)^2}.
\end{equation*}

The marginal score function $\nabla_{x_t} \log p_t(x_t | c)$ can be expressed by marginalizing over the terminal data distribution $p_1(x_1 | c)$ and applying Bayes' theorem:
\begin{equation*}            
\begin{split}
    \nabla_{x_t} \log p_t(x_t | c) &= \frac{1}{p_t(x_t | c)} \nabla_{x_t} p_t(x_t | c) \\
    &= \frac{1}{p_t(x_t | c)} \int p_1(x_1 | c) \nabla_{x_t} p(x_t | x_1) dx_1 \\
    &= \int \frac{p(x_t | x_1) p_1(x_1 | c)}{p_t(x_t | c)} \nabla_{x_t} \log p(x_t | x_1) dx_1 \\
    &= \int p(x_1 | x_t, c) \nabla_{x_t} \log p(x_t | x_1) dx_1.
\end{split}
\end{equation*}

Substituting the explicit form of the conditional score function into the integral gives:
\begin{equation*}
\begin{split}
    \nabla_{x_t} \log p_t(x_t | c) &= \int p(x_1 | x_t, c) \left( -\frac{x_t - t x_1}{(1-t)^2} \right) dx_1 \\
    &= -\frac{x_t}{(1-t)^2} + \frac{t}{(1-t)^2} \int p(x_1 | x_t, c) x_1 dx_1.
\end{split}
\end{equation*}

Substituting the posterior mean $\hat{x}_1(x_t, c) = \int p(x_1 | x_t, c) x_1 dx_1$:
\begin{equation*}
    \nabla_{x_t} \log p_t(x_t | c) = -\frac{x_t}{(1-t)^2} + \frac{t}{(1-t)^2} \hat{x}_1(x_t, c).
\end{equation*}

Rearranging the terms to isolate $\hat{x}_1(x_t, c)$ yields Tweedie's formula, connecting the posterior mean estimation directly to the marginal score:
\begin{equation}
    \hat{x}_1(x_t, c) = \frac{x_t}{t} + \frac{(1-t)^2}{t} \nabla_{x_t} \log p_t(x_t | c).
\label{eq:appendix_tweedie_formula}
\end{equation}

\subsection{Proof of Theorem 1}
\label{app:proof_theorem_1}
\textbf{Theorem 1.} \textit{The optimal velocity field $v^*(x_t, c, t)$ can be equivalently expressed as the reference field rectified by an advantage-weighted residual expectation:}
\begin{equation*}
    v^*(x_t, c, t) = v^{\text{old}}(x_t, c, t) + \mathbb{E}_{x_1 \sim p^{\text{old}}(\cdot | x_t, c)} \left[ \mathcal{A}(x_1, x_t) \left( u_t(x_t | x_1) - v^{\text{old}}(x_t, c, t) \right) \right].
\end{equation*}
\textit{where $\mathcal{A}(x_1, x_t)$ is the Centered Advantage, constructed by leveraging the zero-expectation property $\mathbb{E}[u_t - v^\text{old}] = 0$ to subtract a constant baseline of $1$:}
\begin{equation*}
    \mathcal{A}(x_1, x_t) = \frac{\exp\left(\frac{R(c, x_1)}{\gamma}\right)}{\phi(x_t, c, t)} - 1.
\end{equation*}

\begin{proof}
From Eq.~\eqref{eq:value_gradient_velocity}, the optimal velocity field is strictly determined by the score function of the intermediate optimal marginal:
\begin{equation}
    v^*(x_t, c, t) = v^{\text{old}}(x_t, c, t) + \frac{1-t}{t} \nabla_{x_t} \log \phi(x_t, c, t).
    \label{eq:base_optimal_velocity}
\end{equation}

Expanding the gradient of the log value function and applying the logarithmic derivative identity $\nabla_{x_t} p = p \nabla_{x_t} \log p$, the gradient is derived as:
\begin{align}
    \nabla_{x_t} \log \phi(x_t, c, t) &= \frac{\nabla_{x_t} \phi(x_t, c, t)}{\phi(x_t, c, t)} \nonumber \\
    &= \frac{1}{\phi(x_t, c, t)} \nabla_{x_t} \int p^{\text{old}}(x_1 | x_t, c) \exp\left(\frac{R(c, x_1)}{\gamma}\right) dx_1 \nonumber \\
    &= \frac{1}{\phi(x_t, c, t)} \int p^{\text{old}}(x_1 | x_t, c) \exp\left(\frac{R(c, x_1)}{\gamma}\right) \nonumber \\
    &\quad \cdot \nabla_{x_t} \log p^{\text{old}}(x_1 | x_t, c) dx_1.
    \label{eq:appendix_grad_log_phi}
\end{align}

Relating the score of the posterior $\nabla_{x_t} \log p^{\text{old}}(x_1 | x_t, c)$ to the forward transition fields via Bayes' theorem factorizes the gradient into a direct linear relationship with the velocity residual:
\begin{align}
    \nabla_{x_t} \log p^{\text{old}}(x_1 | x_t, c) &= \nabla_{x_t} \log p(x_t | x_1) - \nabla_{x_t} \log p^{\text{old}}_t(x_t | c) \nonumber \\
    &= \frac{t x_1 - x_t}{(1-t)^2} - \frac{t \hat{x}_1^{\text{old}} - x_t}{(1-t)^2} \nonumber \\
    &= \frac{t}{1-t} \left( \frac{x_1 - \hat{x}_1^{\text{old}}}{1-t} \right) \nonumber \\
    &= \frac{t}{1-t} \left( u_t(x_t | x_1) - v^{\text{old}}(x_t, c, t) \right),
    \label{eq:appendix_score_velocity_link}
\end{align}
where the conditional and marginal scores are derived from Tweedie's formula in Sec \ref{subsec:derivation_tweedie}, while $u_t(x_t | x_1) = \frac{x_1 - x_t}{1-t}$ and $v^{\text{old}}(x_t, c, t) = \frac{\hat{x}_1^{\text{old}} - x_t}{1-t}$ denote the conditional flow target and reference velocity, respectively.

Substituting Eq.\eqref{eq:appendix_score_velocity_link} into Eq.\eqref{eq:appendix_grad_log_phi} transforms the integral into an expectation over the posterior:
\begin{align*}
    \nabla_{x_t} \log \phi(x_t, c, t) &= \frac{t}{1-t} \mathbb{E}_{x_1 \sim p^{\text{old}}(\cdot|x_t, c)} \Bigg[ \\
    &\quad \frac{\exp\left(\frac{R(c, x_1)}{\gamma}\right)}{\phi(x_t, c, t)} \left( u_t(x_t | x_1) - v^{\text{old}}(x_t, c, t) \right) \Bigg].
\end{align*}

Substituting this explicit gradient form back into the base velocity Eq.\eqref{eq:base_optimal_velocity}:
\begin{equation*}
\begin{split}
    v^*(x_t, c, t) =& v^{\text{old}}(x_t, c, t) \\
    &+ \mathbb{E}_{x_1 \sim p^{\text{old}}(\cdot|x_t, c)} \left[ \frac{\exp\left(\frac{R(c, x_1)}{\gamma}\right)}{\phi(x_t, c, t)} \left( u_t(x_t | x_1) - v^{\text{old}}(x_t, c, t) \right) \right].
\end{split}
\end{equation*}

Since the velocity residual has zero expectation, $\mathbb{E}_{x_1 \sim p^{\text{old}}(\cdot|x_t, c)} [ u_t(x_t | x_1) - v^{\text{old}}(x_t, c, t) ] = 0$, subtracting a unit baseline from the weight preserves the equality:
\begin{align*}
    v^*(x_t, c, t) &= v^{\text{old}}(x_t, c, t) \\
    & + \mathbb{E}_{x_1 \sim p^{\text{old}}(\cdot|x_t, c)} \left[ \left( \frac{\exp\left(\frac{R(c, x_1)}{\gamma}\right)}{\phi(x_t, c, t)} - 1 \right) \left( u_t(x_t | x_1) - v^{\text{old}}(x_t, c, t) \right) \right].
\end{align*}

Defining the centered advantage as $\mathcal{A}(x_1, x_t) = \frac{\exp(R(c, x_1)/\gamma)}{\phi(x_t, c, t)} - 1$ explicitly yields the exact formulation of the optimal velocity field, completing the proof.
\end{proof}

\subsection{Proof of Theorem 2}
\label{app:proof_theorem_2}

\textbf{Theorem 2.} \textit{The gradients of the Advantage-Weighted Rectification Loss $\mathcal{L}_{\text{AWR}}$ and the Ideal Loss $\mathcal{L}_{\text{Ideal}}$ are identical with respect to $\theta$:}
\begin{equation*}
    \nabla_\theta \mathcal{L}_{\text{AWR}}(\theta) = \nabla_\theta \mathcal{L}_{\text{Ideal}}(\theta).
\label{eq:appendix_theorem2_statement}
\end{equation*}

\begin{proof}
Define the stochastic target derived from Theorem 1 as $Y(x_1, x_t, c, t) = v^{\text{old}}(x_t, c, t) + \mathcal{A}(x_1, x_t) \left( u_t(x_t | x_1) - v^{\text{old}}(x_t, c, t) \right)$.
Theorem~\ref{thm:advantage_rectification} establishes that the theoretical optimal velocity field is the conditional expectation of this stochastic target: $v^*(x_t, c, t) = \mathbb{E}_{x_1 \sim p^{\text{old}}(\cdot|x_t, c)} [Y(x_1, x_t, c, t)]$.

The Advantage-Weighted Rectification loss is formulated as:
\begin{equation*}
    \mathcal{L}_{\text{AWR}}(\theta) = \mathbb{E}_{t, c, x_t} \left[ \mathbb{E}_{x_1 \sim p^{\text{old}}(\cdot|x_t, c)} \left[ \| v_\theta(x_t, c, t) - Y(x_1, x_t, c, t) \|^2 \right] \right].
\end{equation*}

Expanding the squared $L_2$ norm by introducing the optimal field $v^*(x_t, c, t)$ yields a standard bias-variance decomposition:
\begin{align*}
    \| v_\theta - Y \|^2 &= \| (v_\theta - v^*) + (v^* - Y) \|^2 \\
    &= \| v_\theta - v^* \|^2 + \| v^* - Y \|^2 + 2 \langle v_\theta - v^*, v^* - Y \rangle.
\end{align*}

Substituting this decomposition into the expectation separates the loss function into three distinct terms:
\begin{align*}
    \mathcal{L}_{\text{AWR}}(\theta) &= \mathbb{E}_{t, c, x_t} \left[ \| v_\theta(x_t, c, t) - v^*(x_t, c, t) \|^2 \right] \nonumber \\
    &\quad + \mathbb{E}_{t, c, x_t, x_1} \left[ \| v^*(x_t, c, t) - Y(x_1, x_t, c, t) \|^2 \right] \nonumber \\
    &\quad + \mathbb{E}_{t, c, x_t} \left[ 2 \langle v_\theta - v^*, \mathbb{E}_{x_1 | x_t, c} [v^* - Y] \rangle \right].
    \label{eq:appendix_loss_decomposition}
\end{align*}

The three terms evaluate as follows:
\begin{itemize}
    \item The first term is strictly the Ideal Loss $\mathcal{L}_{\text{Ideal}}(\theta)$, which measures the regression error between the parameterized velocity and the theoretical optimal field.
    \item The second term represents the irreducible variance of the stochastic target $Y$. Because this term depends solely on the data and the reference policy, it is independent of the model parameters $\theta$, yielding a gradient of zero.
    \item In the third term, the inner expectation evaluates to $\mathbb{E}_{x_1 | x_t, c} [v^* - Y] = v^* - \mathbb{E}_{x_1 | x_t, c}[Y]$. By the definition of $v^*$, this equals $v^* - v^* = 0$, causing the entire cross-term to vanish.
\end{itemize}

Consequently, the AWR loss simplifies to the Ideal Loss plus an independent constant term $C$:
\begin{equation*}
    \mathcal{L}_{\text{AWR}}(\theta) = \mathcal{L}_{\text{Ideal}}(\theta) + C.
\end{equation*}

Taking the gradient with respect to $\theta$ eliminates the constant $C$, proving the exact optimization equivalence:
\begin{equation*}
    \nabla_\theta \mathcal{L}_{\text{AWR}}(\theta) = \nabla_\theta \mathcal{L}_{\text{Ideal}}(\theta).
\end{equation*}
\end{proof}

\section{Theoretical Discussions with RWR and DPO}
\label{sec:theoretical_discussions}

RWR and DPO represent distinct alignment paradigms that originate from the same theoretical optimal policy. Despite their divergent practical implementations, both methods can be mathematically unified under the KL-constrained reward maximization framework, leveraging the closed-form optimal density derived in Eq.\eqref{eq:appendix_optimal_policy_final}.

\subsection{Connection to RWR in Flow Matching}
\label{subsec:discussion_rwr_flow}

Under the ReFlow schedule $x_t = t x_1 + (1 - t) x_0$, aligning a parameterized velocity $v_\theta(x_t, c, t)$ with the ideal marginal velocity $v^*(x_t, c, t) = \mathbb{E}_{x_1 \sim p^*(\cdot | x_t, c)} [ u_t(x_t | x_1) ]$ yields the global matching objective:
\begin{equation*}
    \mathcal{L}_{\text{FM}}(\theta) = \mathbb{E}_{t, x_t \sim p_t^*} \left[ \| v_\theta(x_t, c, t) - v^*(x_t, c, t) \|^2 \right].
\end{equation*}

To bypass the intractable marginal expectation in $v^*$, Conditional Flow Matching (CFM) establishes a gradient equivalence. By expanding the squared norm and marginalizing the cross-term $\mathbb{E}_{x_t \sim p_t^*} [\langle v_\theta, v^* \rangle]$, the objective reduces to matching the conditional field:
\begin{align*}
    \iint p_1^*(x_1|c) p(x_t|x_1) \langle v_\theta, u_t \rangle dx_1 dx_t = \mathbb{E}_{x_1 \sim p_1^*} \mathbb{E}_{x_t \sim p(\cdot|x_1)} [\langle v_\theta, u_t \rangle].
\end{align*}
which recovers the conditional matching objective over the optimal target distribution:
\begin{equation*}
    \mathcal{L}_{\text{CFM}}(\theta) = \mathbb{E}_{t, x_1 \sim p_1^*, x_0 \sim p_0} \left[ \| v_\theta(x_t, c, t) - (x_1 - x_0) \|^2 \right].
\end{equation*}

Since sampling directly from $p_1^*$ is practically infeasible, applying importance sampling from the accessible reference policy $p_1^{\text{old}}$ re-weights the expectation by the density ratio $\frac{p_1^*(x_1|c)}{p_1^{\text{old}}(x_1|c)} \propto \exp\left(\frac{R(c, x_1)}{\gamma}\right)$. Absorbing the partition function $Z(c)$ establishes the final RWR objective:
\begin{equation*}
    \mathcal{L}_{\text{RWR}}(\theta) \propto \mathbb{E}_{x_1 \sim p_1^{\text{old}}} \left[ \exp\left(\frac{R(c, x_1)}{\gamma}\right) \mathbb{E}_{t, x_0 \sim p_0} \left[ \| v_\theta(x_t, c, t) - (x_1 - x_0) \|^2 \right] \right].
\end{equation*}
which bridges the classical RWR alignment paradigm with continuous generative processes, demonstrating that exponentiated reward re-weighting naturally emerges from projecting the velocity onto the optimal policy trajectory.

\begin{figure}[!t]
    \centering
    \includegraphics[width=\linewidth]{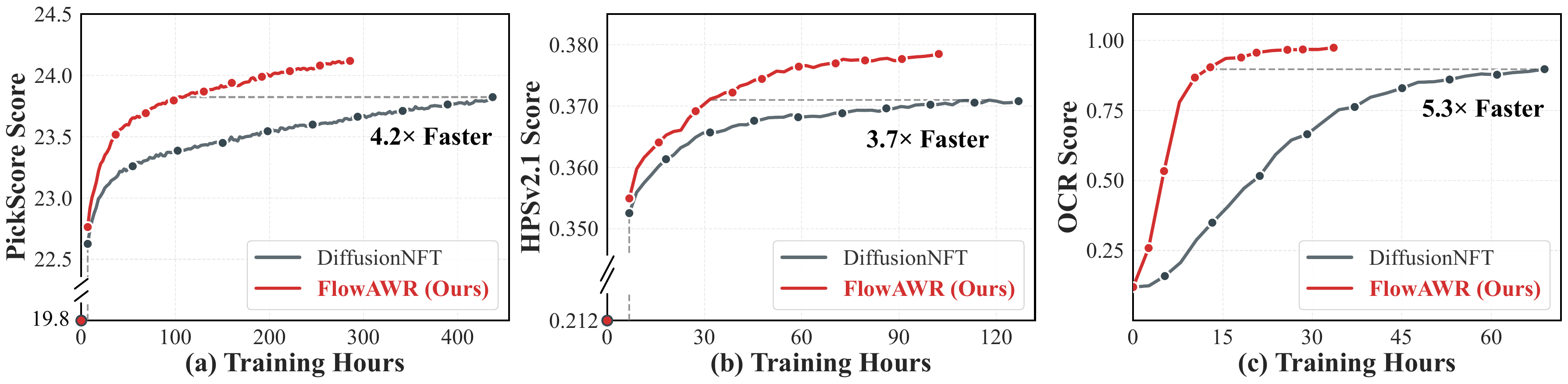}\par
    \vspace{-4pt}
    \caption{FlowAWR versus DiffusionNFT on (a) PickScore, (b) HPSv2.1, and (c) OCR, against GPU-hours. Dashed horizontal lines compare the cost of reaching DiffusionNFT's highest recorded reward under the plotted configuration.}
    \label{fig:nft_flowawr_gpu_hours}
    \vspace{-0.1in}
\end{figure}

\subsection{Connection to DPO}
\label{subsec:discussion_dpo}

Unlike RWR's regression approach, DPO bypasses explicit reward modeling. Rearranging the logarithmic form of the optimal policy analytically reparameterizes the unobserved reward $R(s, a)$ strictly in terms of policy probabilities:
\begin{equation*}
    R(s, a) = \gamma \log \frac{\pi^*(a|s)}{\pi^{\text{old}}(a|s)} + \gamma \log Z(s).
\end{equation*}

Under the Bradley-Terry (BT) preference model, the probability that an action $y_w$ is preferred over $y_l$ is defined by the logistic function of their reward difference, $P(y_w \succ y_l | s) = \sigma \left( R(s, y_w) - R(s, y_l) \right)$. Substituting the reparameterized rewards perfectly cancels the intractable partition function $Z(s)$:
\begin{equation*}
    R(s, y_w) - R(s, y_l) = \gamma \log \frac{\pi^*(y_w|s)}{\pi^{\text{old}}(y_w|s)} - \gamma \log \frac{\pi^*(y_l|s)}{\pi^{\text{old}}(y_l|s)}.
\end{equation*}

Parameterizing $\pi^*$ with a learnable policy $\pi_\theta$ and minimizing the negative log-likelihood over a preference dataset $\mathcal{D}$ directly yields the exact DPO objective:
\begin{equation}
    \mathcal{L}_{\text{DPO}}(\theta) = -\mathbb{E}_{(s, y_w, y_l) \sim \mathcal{D}} \left[ \log \sigma \left( \gamma \log \frac{\pi_\theta(y_w|s)}{\pi^{\text{old}}(y_w|s)} - \gamma \log \frac{\pi_\theta(y_l|s)}{\pi^{\text{old}}(y_l|s)} \right) \right].
\label{eq:appendix_dpo_loss}
\end{equation}
which confirms that DPO structurally inherits the KL-constrained optimality criteria, simply reformulating it to accommodate pairwise preference data.

\section{Experiment Details}
\label{sec:experiment_details}

\subsection{Training Cost and Additional Baseline Comparisons}
\label{subsec:training_cost}

We report sampling-and-training cost in GPU-hours. The horizontal axes in Figures~\ref{fig:nft_flowawr_gpu_hours} and~\ref{fig:additional_baselines_gpu_hours}, labeled ``Training Hours,'' are computed by multiplying each method's per-update time by its update count and eight GPUs, then dividing by 3,600. DiffusionNFT uses its own timing measurements in both figures. The FlowAWR--DiffusionNFT comparison concerns single-reward training with 10-step rollouts and 40-step ODE evaluation.

\paragraph{Computational Cost.}
We report aggregate GPU-hours across eight GPUs, including sampling costs amortized over the associated training updates. For TempFlow-GRPO, we additionally account for 70 seconds of reward computation per training update before aggregation across GPUs. On SD3.5-Medium, FlowAWR requires 0.2549, 0.1292, 0.2285, and 0.2273 GPU-hours per update on GenEval, OCR, PickScore, and HPSv2.1, respectively. Compared with DiffusionNFT, FlowAWR reduces the per-update cost by 2.5\% on OCR, while the differences on the other three tasks remain below 0.4\%. Relative to AWM, FlowAWR reduces the cost by 35.8\%, 31.6\%, and 1.9\% on GenEval, OCR, and PickScore, respectively, while incurring a 0.4\% increase on HPSv2.1. Compared with TempFlow-GRPO, the corresponding reductions are 52.1\%, 82.2\%, 56.7\%, and 57.7\%. For FLUX.1-dev on HPSv3, we instead report the cost of a complete sampling iteration, including all subsequent training updates: one update for FlowAWR and two for TempFlow-GRPO. Under this accounting, FlowAWR requires 1.5909 GPU-hours per sampling iteration, compared with 3.1422 for TempFlow-GRPO, a reduction of 49.4\%.

Figure~\ref{fig:nft_flowawr_gpu_hours} compares FlowAWR with DiffusionNFT. Its speedup annotations use DiffusionNFT's cost to reach its highest recorded reward divided by FlowAWR's cost at the first evaluated checkpoint reaching that reward. The OCR curves use DiffusionNFT's general configuration, rather than the OCR-specific configuration in Table~\ref{tab:baseline_comparison_grouped}; Appendix~\ref{subsec:nft_ocr_configuration} documents this distinction.

\paragraph{Additional Baselines.}
We further compare FlowAWR with AWM~\citep{xue2025advantage}, DiffusionNFT~\citep{zheng2025diffusionnft}, and TempFlow-GRPO~\citep{he2026tempflow} on three single-reward optimization tasks: PickScore, HPSv2.1, and OCR. Figure~\ref{fig:additional_baselines_gpu_hours} reports their learning curves against training cost measured in GPU-hours.

\begin{figure}[htbp]
    \centering
    \includegraphics[width=\linewidth]{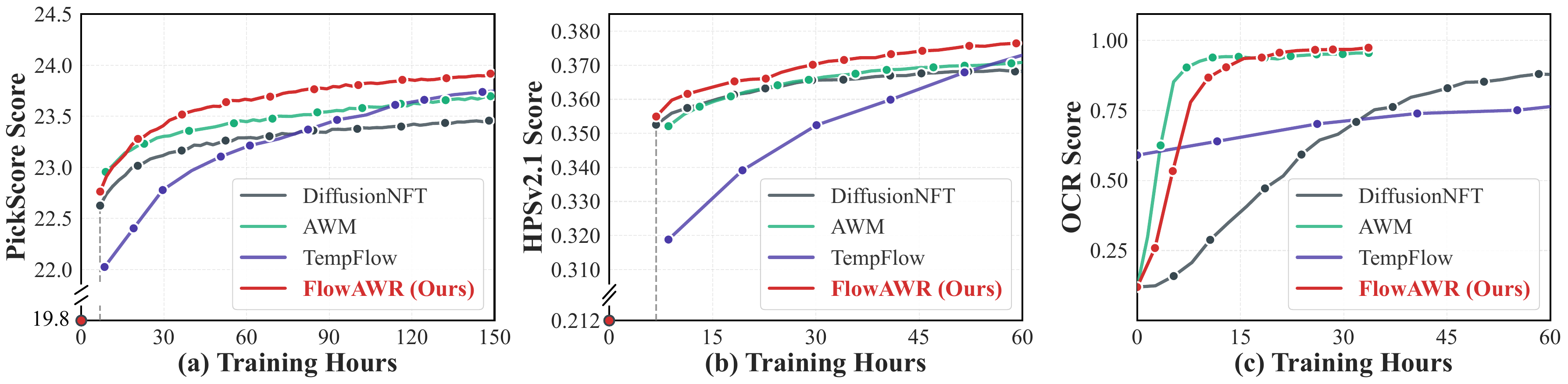}\par
    \vspace{-4pt}
    \caption{Single-reward learning curves for FlowAWR, AWM, DiffusionNFT, and TempFlow-GRPO on (a) PickScore, (b) HPSv2.1, and (c) OCR. Training cost is measured in GPU-hours.}
    \label{fig:additional_baselines_gpu_hours}
    \vspace{-0.1in}
\end{figure}

\subsection{Experiments on FLUX.1-dev}
\label{subsec:flux_experiments}

\noindent
\begin{minipage}[t]{0.49\textwidth}
    \vspace{0pt}
    We further evaluate FlowAWR on FLUX.1-dev against DiffusionNFT and TempFlow-GRPO using the HPDv2 dataset, consistent with TempFlow-GRPO~\citep{he2026tempflow}. All methods use a training resolution of $720\times720$ and an evaluation resolution of $1024\times1024$. Following the rollout-only CFG strategy of \citet{xu2026qwen}, we enable CFG during rollout sampling but disable it when computing the training loss. CFG is also enabled during evaluation. Figure~\ref{fig:flux_hpsv3} compares HPSv3~\citep{ma2025hpsv3} learning curves against GPU-hours. FlowAWR reaches a higher HPSv3 score than both baselines within the displayed training range.
\end{minipage}\hfill
\begin{minipage}[t]{0.45\textwidth}
    \vspace{0pt}
    \captionsetup{type=figure}
    \centering
    \includegraphics[width=\linewidth]{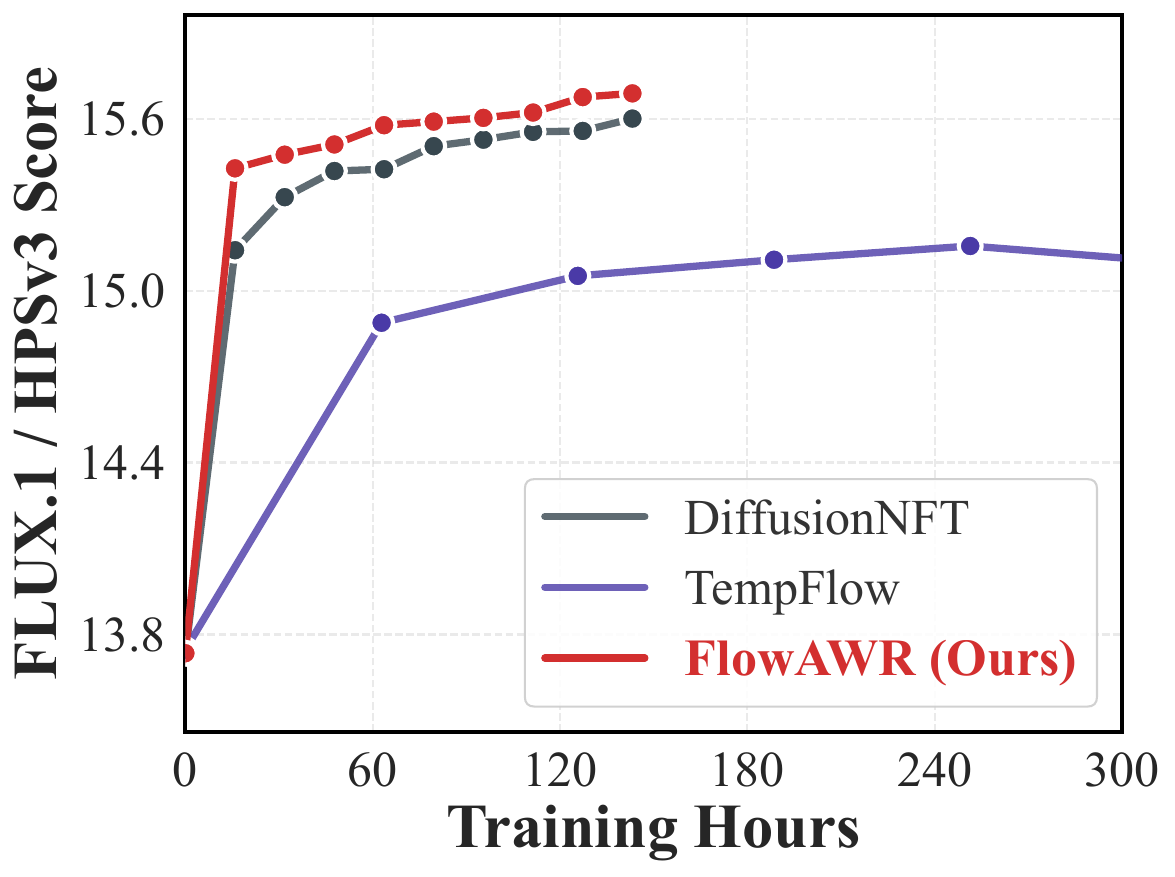}\par
    \vspace{-0.1in}
    \caption{HPSv3 learning curves for FlowAWR, DiffusionNFT, and TempFlow-GRPO on FLUX.1-dev using HPDv2. Training cost is measured in GPU-hours.}
    \label{fig:flux_hpsv3}
\end{minipage}\par

\subsection{Empirical Range of the Centered Advantage}

\begin{figure*}[htbp]
    \centering
    \vspace{-0.1in}
    \begin{subfigure}{0.32\linewidth}
        \centering
        \includegraphics[width=\linewidth]{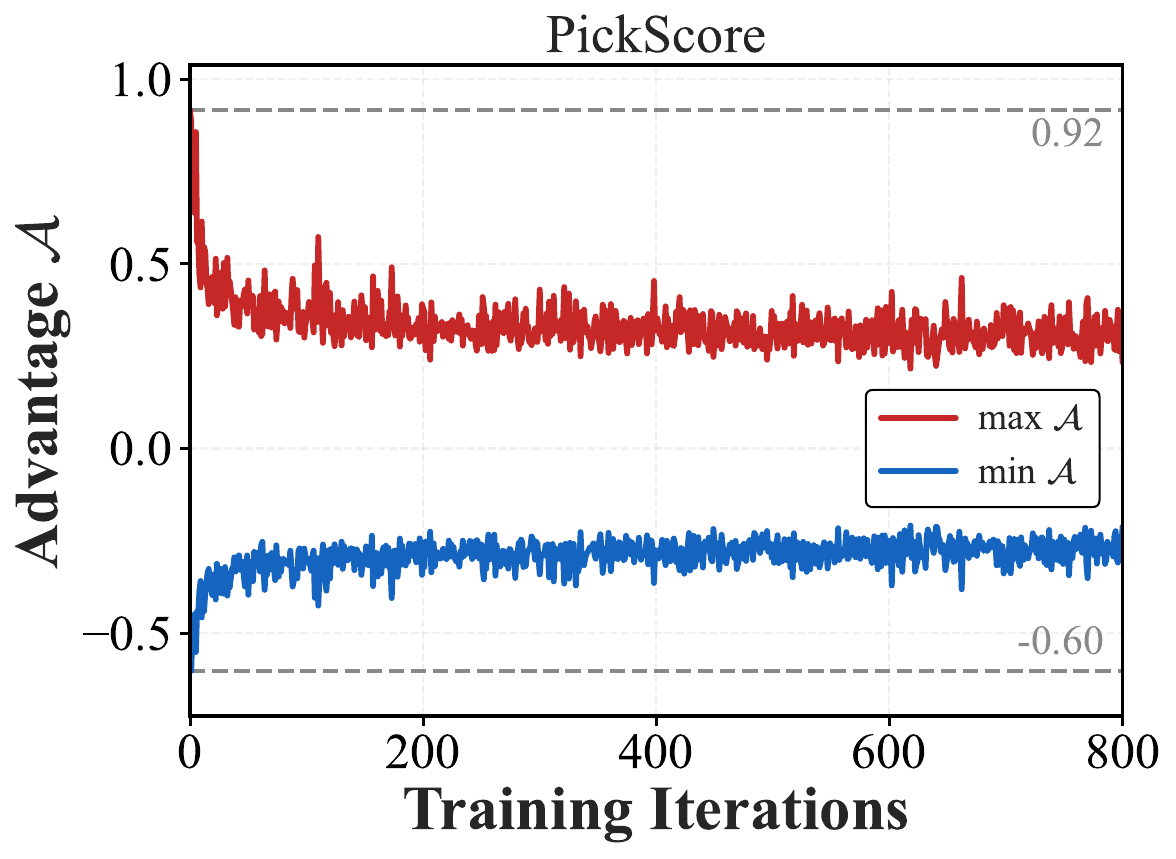}
        \caption{PickScore}
        \label{fig:adv_pickscore}
    \end{subfigure}\hfill
    \begin{subfigure}{0.32\linewidth}
        \centering
        \includegraphics[width=\linewidth]{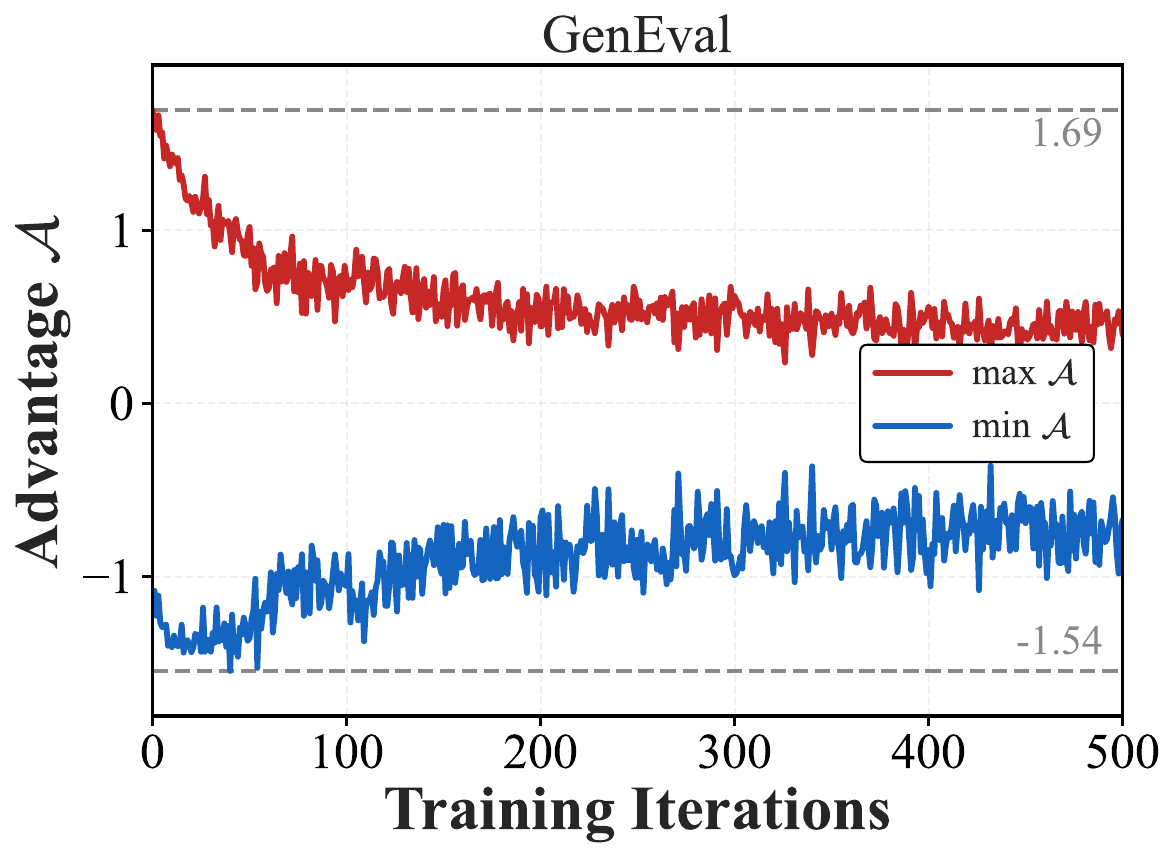}
        \caption{GenEval}
        \label{fig:adv_geneval}
    \end{subfigure}\hfill
    \begin{subfigure}{0.32\linewidth}
        \centering
        \includegraphics[width=\linewidth]{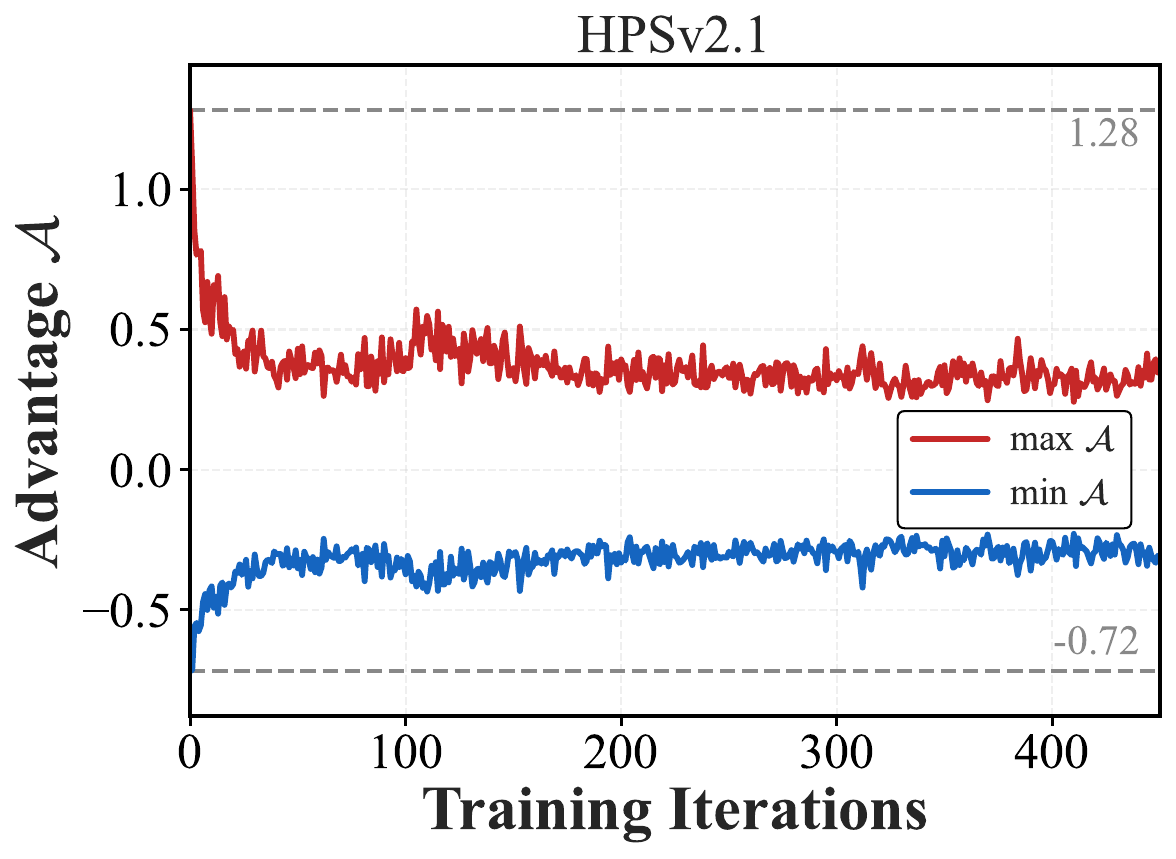}
        \caption{HPSv2.1}
        \label{fig:adv_hpsv2}
    \end{subfigure}
    
    \vspace{0.1in} % 两行之间的垂直间距
    
    % 第二行：OCR + 2个多奖励任务
    \begin{subfigure}{0.32\linewidth}
        \centering
        \includegraphics[width=\linewidth]{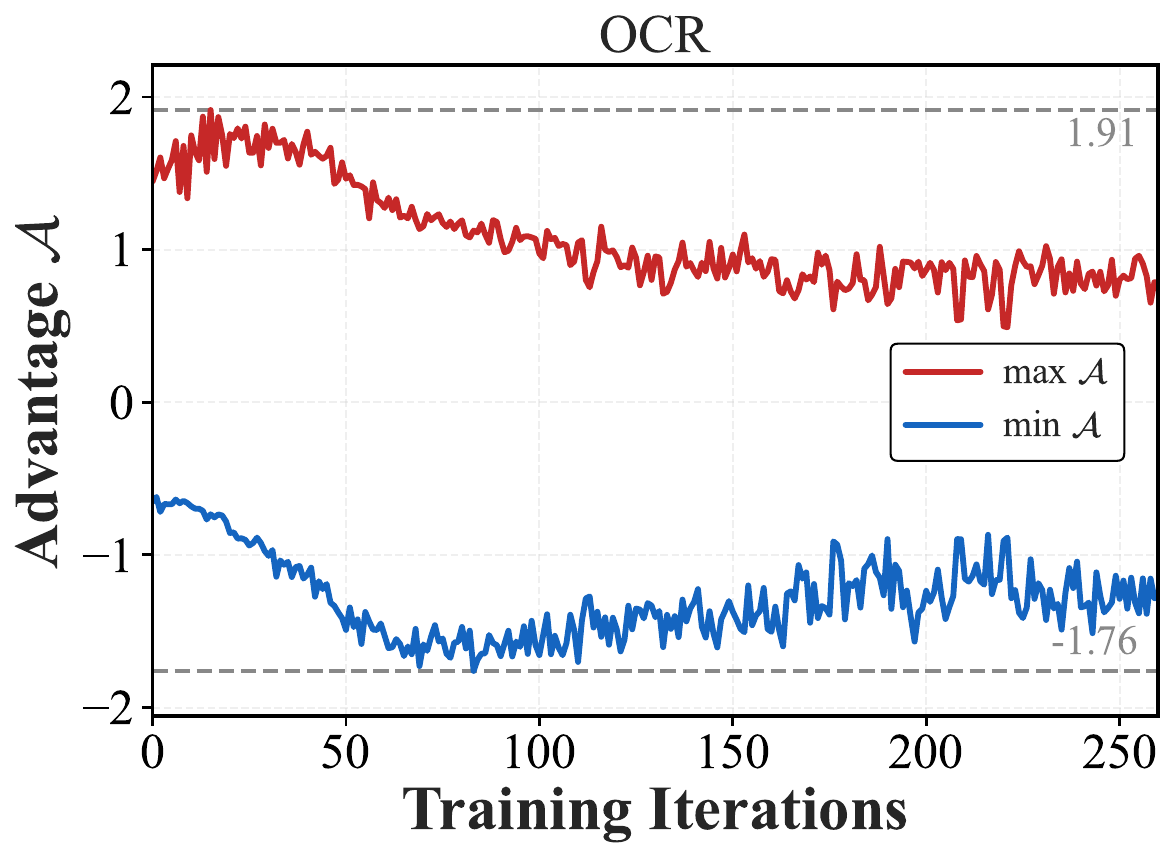}
        \caption{OCR}
        \label{fig:adv_ocr}
    \end{subfigure}\hfill
    \begin{subfigure}{0.32\linewidth}
        \centering
        \includegraphics[width=\linewidth]{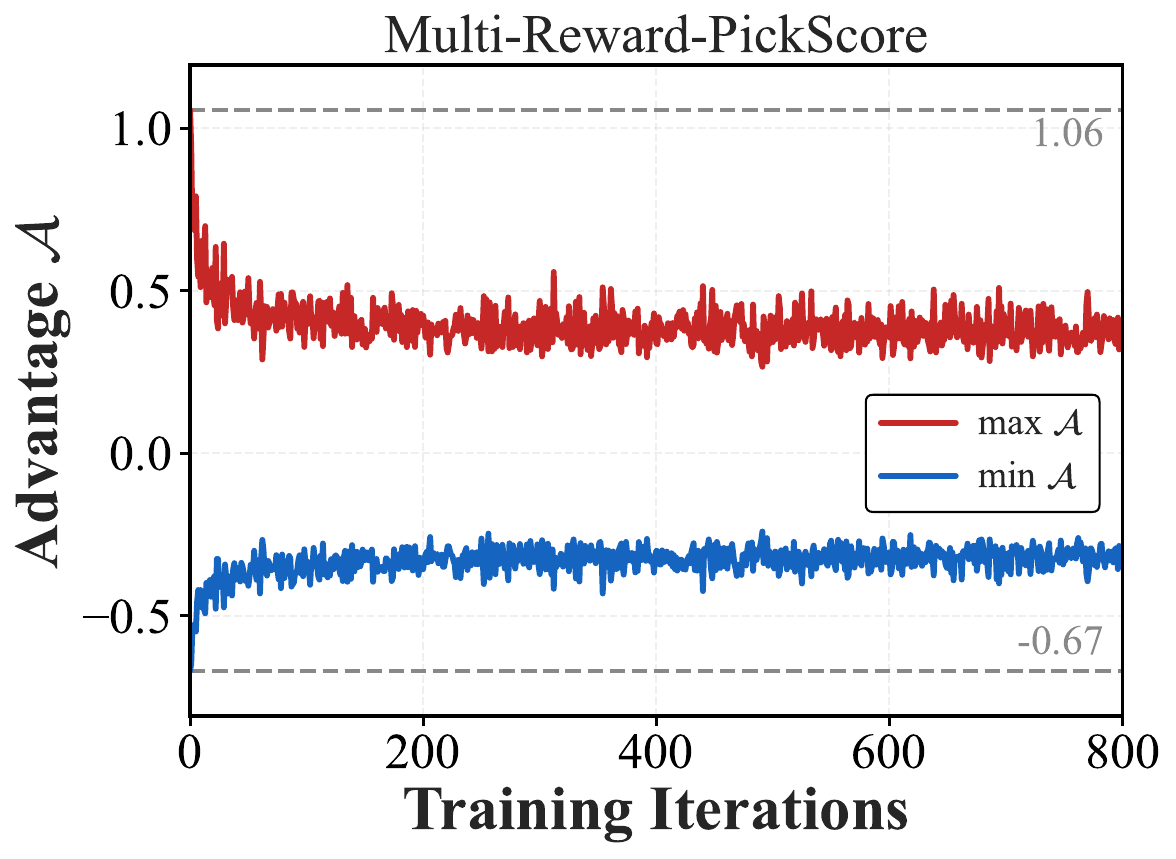}
        \caption{Multi-Reward (PickScore)}
        \label{fig:adv_multi_pickscore}
    \end{subfigure}\hfill
    \begin{subfigure}{0.32\linewidth}
        \centering
        \includegraphics[width=\linewidth]{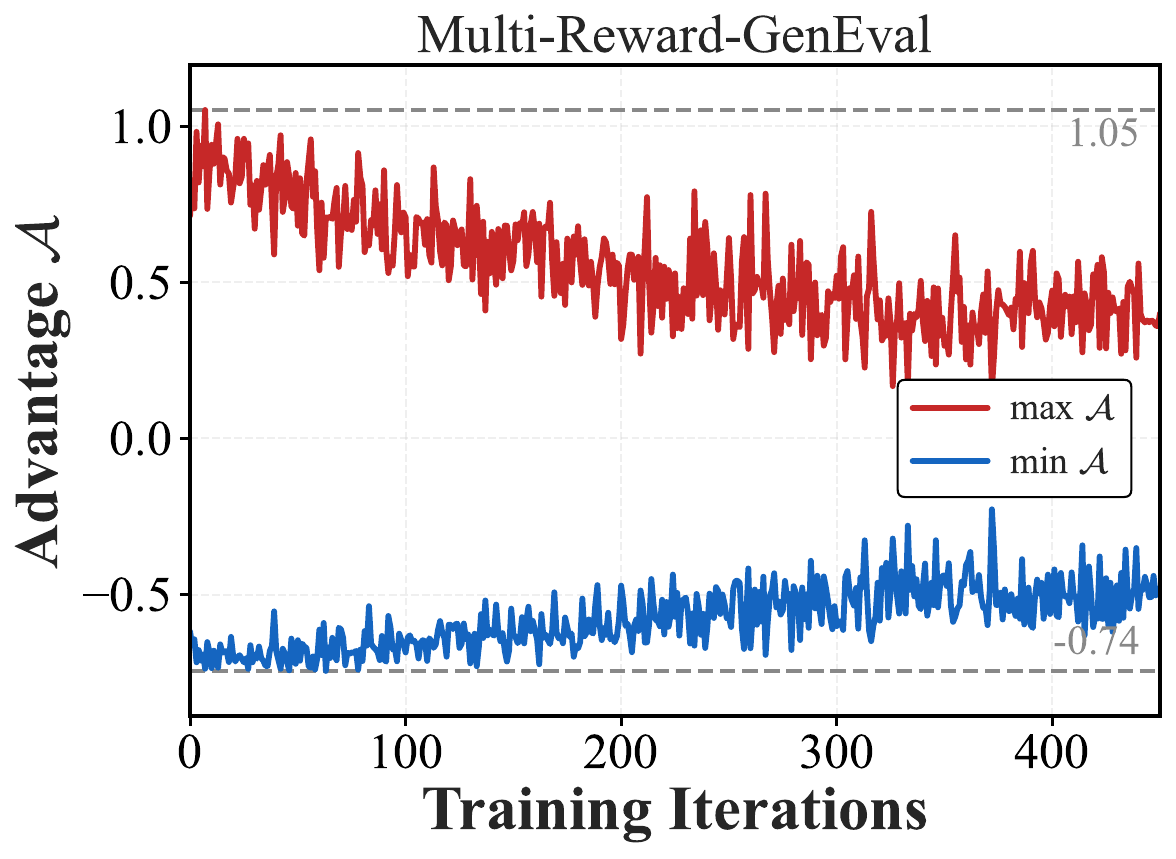}
        \caption{Multi-Reward (GenEval)}
        \label{fig:adv_multi_geneval}
    \end{subfigure}
    
    \caption{Empirical variation of the centered advantage $\mathcal{A}(x_1, x_t)$ across six alignment tasks. Under both single-reward and multi-reward settings, $\mathcal{A}(x_1, x_t)$ remains structurally bounded. It provides discriminative learning signals while naturally restricting the upper and lower bound.}
    \label{fig:advantage_6_tasks}
    \vspace{-0.1in}
\end{figure*}

To validate the numerical stability of the proposed framework, Figure \ref{fig:advantage_6_tasks} illustrates the empirical variation range of the centered advantage $\mathcal{A}(x_1, x_t)$ across six distinct alignment tasks, encompassing both single-reward (PickScore, GenEval, HPSv2.1, OCR) and multi-reward optimizations. While the advantage values exhibit the necessary fluctuations to provide discriminative learning signals, their overall dynamic range remains strictly bounded and controllable throughout the training process. The effective coefficients remain bounded under the task-specific scaling, which allows values below $-1$ for GenEval and OCR. These results illustrate how reward normalization and task-specific scaling control coefficient magnitudes across tasks without directly clipping the advantages.

\subsection{Experimental Setup}
\label{subsec:appendix_experimental_setup}

\paragraph{Training Configurations.} 
The training setup follows FlowGRPO and DiffusionNFT, which adopts 48 groups per epoch, a group size of 24, and a LoRA parameterization ($r=32$, $\alpha=64$) with a constant learning rate of $3 \times 10^{-4}$. For each collected clean image, forward noising and loss computation are performed on the corresponding sampling timesteps. Data collection employs a 2nd-order ODE sampler with adaptive time weighting by default.

\paragraph{Single-Reward Optimization.} 
For fair comparison with DiffusionNFT, we use the same 10 sampling steps and the shared EMA decay schedule $\mu_i=\min(0.001i,0.5)$ across the single-reward tasks, while DiffusionNFT uses its reported configurations. HPSv2.1 optimization uses the same Pick-a-Pic dataset as PickScore.

\paragraph{Advantage Scaling.}
We compute within-prompt softmax probabilities $p_i$ from reward logits clipped to $[-5,5]$, using max-subtraction and float64 arithmetic. The raw advantage is $\mathcal{A}_i^{\mathrm{raw}}=p_i-1/G$. PickScore and HPSv2.1 use $G\mathcal{A}_i^{\mathrm{raw}}$, following Eq.~\eqref{eq:Softmax_Advantage} with clipped logits. For rule-based rewards such as GenEval and OCR, closely clustered rewards within a prompt group can yield small group-relative advantages. We therefore rescale the centered advantages by their within-group standard deviation to adapt the correction magnitude, using a lower-bounded denominator to avoid excessive amplification.

\paragraph{Single-Branch Advantage Ablation.}
The fixed-coefficient variants in Figure~\ref{fig:ablation_study}(d,e) use $\beta\in\{1,0.1\}$ and select one NFT-style target according to the sign of each sample's advantage. Positive advantages select $v^\old+(u_t-v^\old)/\beta$ with outer loss weight $r$, while negative advantages select $v^\old-(u_t-v^\old)/\beta$ with weight $1-r$, where $r$ is the normalized reward used in the NFT loss. Each sample contributes only its selected branch. In additional trials, omitting these outer reward weights also led to reward collapse.

\paragraph{Multi-Reward Optimization.}
Following the multi-stage curriculum of DiffusionNFT, the training is structured around three dataset-specific configurations: (1) \textbf{3 rewards} (PickScore, CLIPScore, and HPSv2.1) on the Pick-a-Pic dataset; (2) \textbf{4 rewards} (the previous three plus GenEval) on the GenEval dataset; and (3) \textbf{4 rewards} (the initial three plus OCR) on the OCR dataset. To overcome the low quality of initial CFG-free generations, the curriculum sequentially applies configuration (1) for 800 iterations, (2) for 200 iterations, (1) for 200 iterations, (2) for 200 iterations, and finally concludes with (3) for 100 iterations. 
Within each stage, all active rewards are equally weighted, with PickScore explicitly divided by 26 for $[0, 1]$ normalization. For data collection, we maintain the sampling steps as 10 to ensure training efficiency. We maintain the shared EMA decay schedule $\mu_i=\min(0.001i,0.5)$ across all stages.

Note that these iteration counts are approximate. To ensure continuous optimization during stage transitions, the complete training state is strictly resumed from the preceding stage, including the model parameters $\theta$, the reference policy $\theta_{\text{old}}$, the optimizer, and the gradient scaler.

\subsection{DiffusionNFT OCR Configurations}
\label{subsec:nft_ocr_configuration}

The OCR checkpoint in Table~\ref{tab:baseline_comparison_grouped} uses DiffusionNFT's official task-specific configuration, whereas the OCR learning curves in Figures~\ref{fig:main_efficiency}(c), \ref{fig:nft_flowawr_gpu_hours}(c), and~\ref{fig:additional_baselines_gpu_hours}(c) use its general GenEval/PickScore configuration. For DiffusionNFT, $\eta_i$ controls the EMA update of the lagging old policy:
\begin{equation*}
    \theta_i^{\old}\leftarrow\eta_i\theta_{i-1}^{\old}+(1-\eta_i)\theta_i.
\end{equation*}
The general configuration uses $\beta=1$ and $\eta_i=\min(0.001i,0.5)$. The OCR-specific configuration jointly sets $\beta=0.1$ and uses the delayed schedule
\begin{equation*}
    \eta_i=
    \begin{cases}
        0, & i<75,\\
        \min\!\left(0.0075(i-75),0.999\right), & i\geq75.
    \end{cases}
\end{equation*}
FlowAWR retains the shared EMA schedule specified in Section~\ref{subsec:experimental_setup}, without OCR-specific adjustment.

Table~\ref{tab:nft_ocr_configuration} compares both DiffusionNFT configurations with FlowAWR, which achieves an OCR score of 0.975 after 260 updates. The general configuration obtains an OCR score of 0.89 after 520 updates, compared with 0.98 after 400 updates for the OCR-specific configuration. This comparison indicates sensitivity to the jointly selected $(\beta,\eta_i)$ configuration, but does not isolate either parameter's individual contribution or their interaction. It compares separate training configurations with different iteration budgets, not a decline within one training run; the non-OCR metrics do not uniformly decline.

\begin{table}[htbp]
    \centering
    \caption{Comparison of FlowAWR with DiffusionNFT under general and OCR-specific configurations for single-reward OCR training. The OCR-specific DiffusionNFT and FlowAWR results are reproduced from Table~\ref{tab:baseline_comparison_grouped}. All runs use 10-step training rollouts and 40-step first-order ODE evaluation.}
    \label{tab:nft_ocr_configuration}
    \vspace{-4pt}
    \resizebox{\textwidth}{!}{%
        \setlength{\tabcolsep}{4pt}
        \begin{tabular}{lccccccccc}
            \toprule
            \multirow{2}{*}{\textbf{Method / Configuration}} & \multirow{2}{*}{\textbf{\#Iter}} & \multicolumn{3}{c}{\textbf{Task Metric}} & \multicolumn{5}{c}{\textbf{DrawBench Metric}} \\
            \cmidrule(lr){3-5} \cmidrule(lr){6-10}
            & & \textbf{GenEval} & \textbf{OCR} & \textbf{PickScore} & \textbf{CLIPScore} & \textbf{HPSv2.1} & \textbf{Aesthetic} & \textbf{ImgRwd} & \textbf{PickScore} \\
            \midrule
            DiffusionNFT (General) & 0.52k & 0.24 & 0.89 & 19.38 & 0.240 & 0.183 & 4.90 & -0.92 & 20.27 \\
            DiffusionNFT (OCR-specific) & 0.4k & 0.20 & 0.98 & 19.15 & 0.233 & 0.168 & 4.79 & -1.27 & 20.15 \\
            \midrule
            FlowAWR & \textbf{0.26k} & 0.29 & 0.975 & 19.43 & 0.246 & 0.193 & 5.00 & -0.78 & 20.51 \\
            \bottomrule
        \end{tabular}%
    }\par
    \vspace{-0.1in}
\end{table}

\section{Additional Qualitative Results}
\label{sec:additional_results}

We provide additional qualitative comparisons between the base model, FlowGRPO, DiffusionNFT, and our proposed method in Figure \ref{fig:qual_geneval}, Figure \ref{fig:qual_ocr}, and Figure \ref{fig:qual_drawbench}, corresponding to the GenEval, OCR, and DrawBench benchmarks, respectively.

\begin{figure*}[t]
    \centering
    \includegraphics[width=0.96\linewidth]{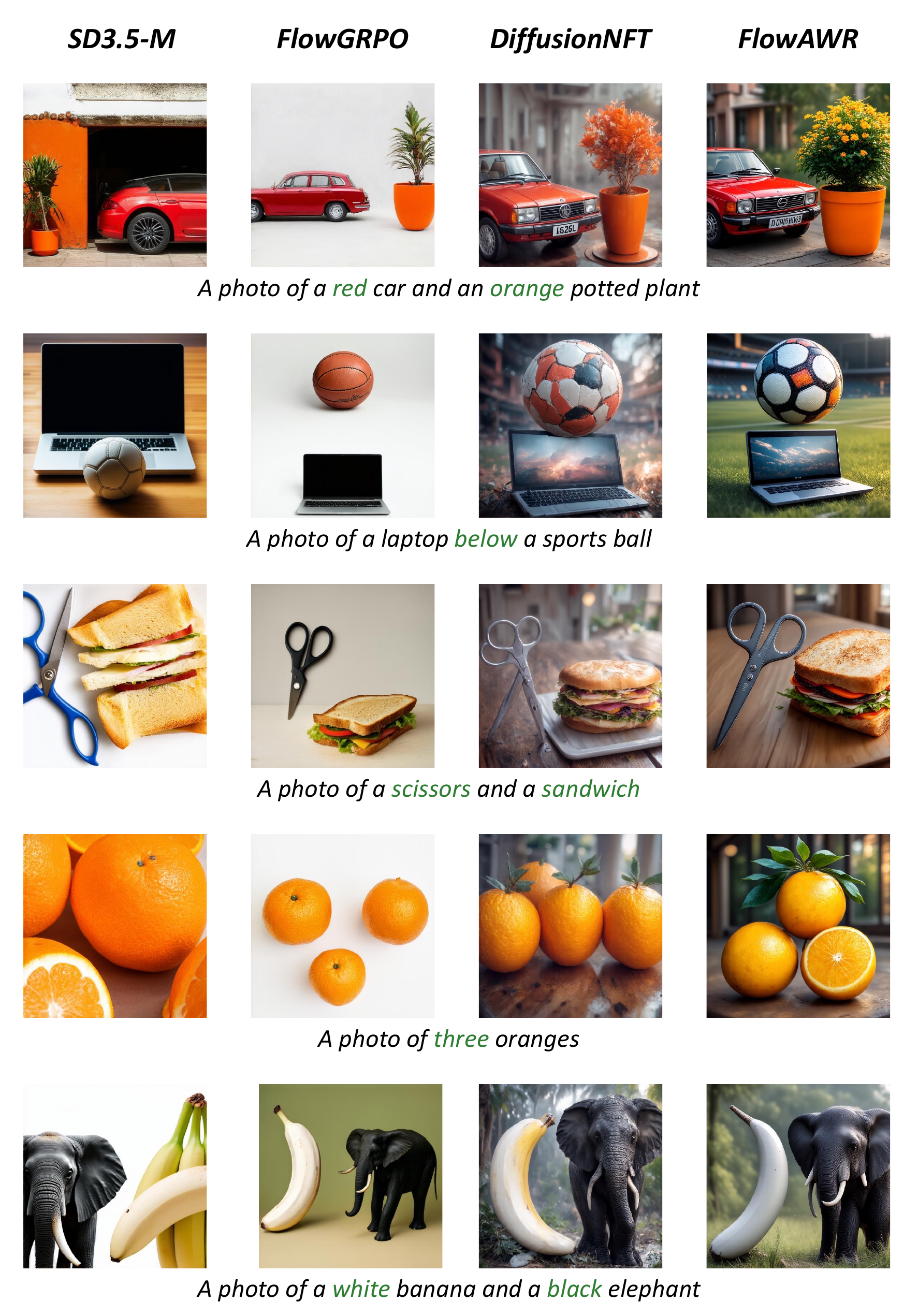}
    \vspace{-0.1in}
    \caption{Qualitative comparison on the GenEval benchmark.}
    \label{fig:qual_geneval}
\end{figure*}

\begin{figure*}[t]
    \centering
    \includegraphics[width=0.96\linewidth]{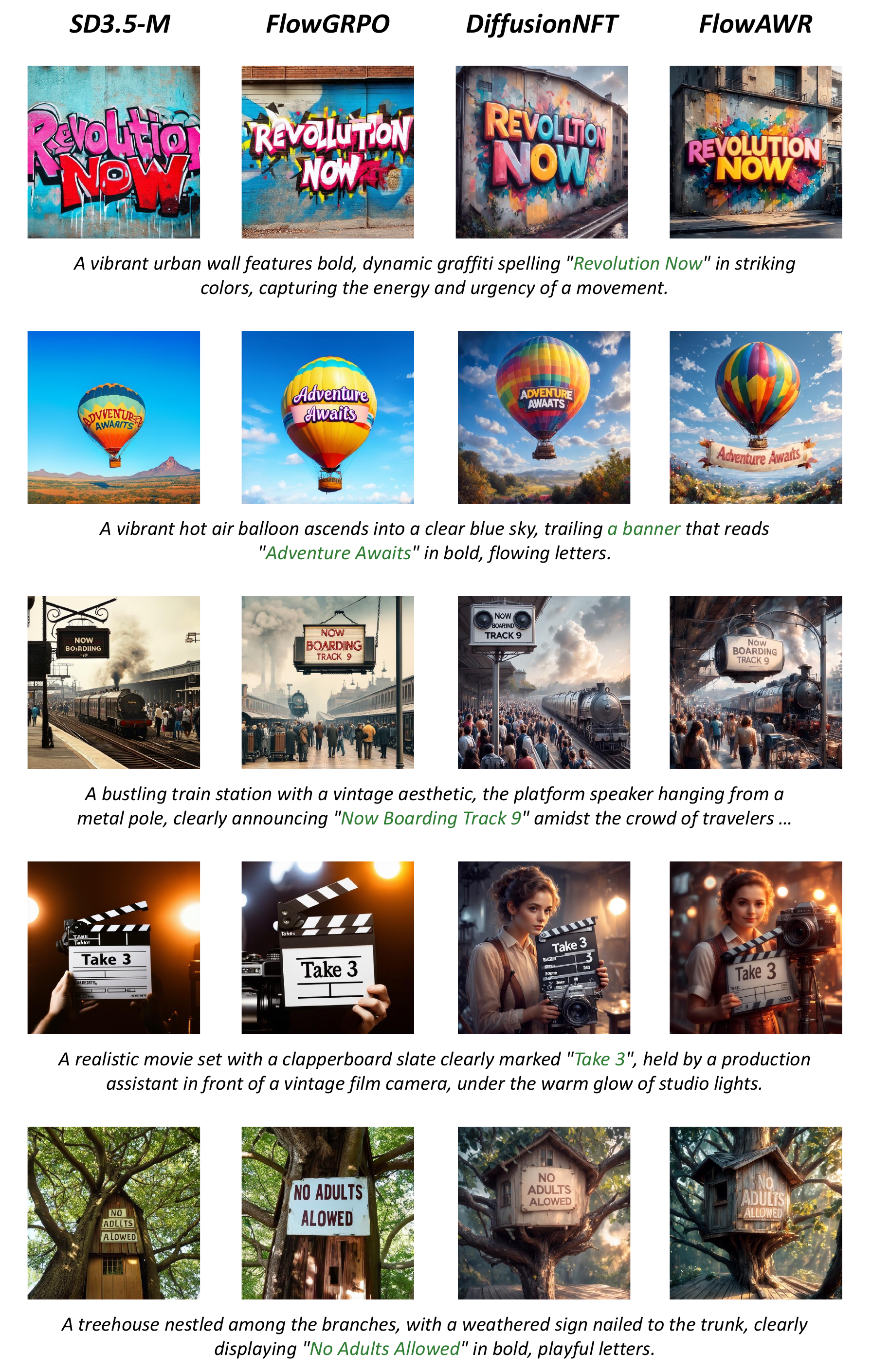}
    \vspace{-0.1in}
    \caption{Qualitative comparison on OCR tasks.}
    \label{fig:qual_ocr}
\end{figure*}

\begin{figure*}[t]
    \centering
    \includegraphics[width=0.96\linewidth]{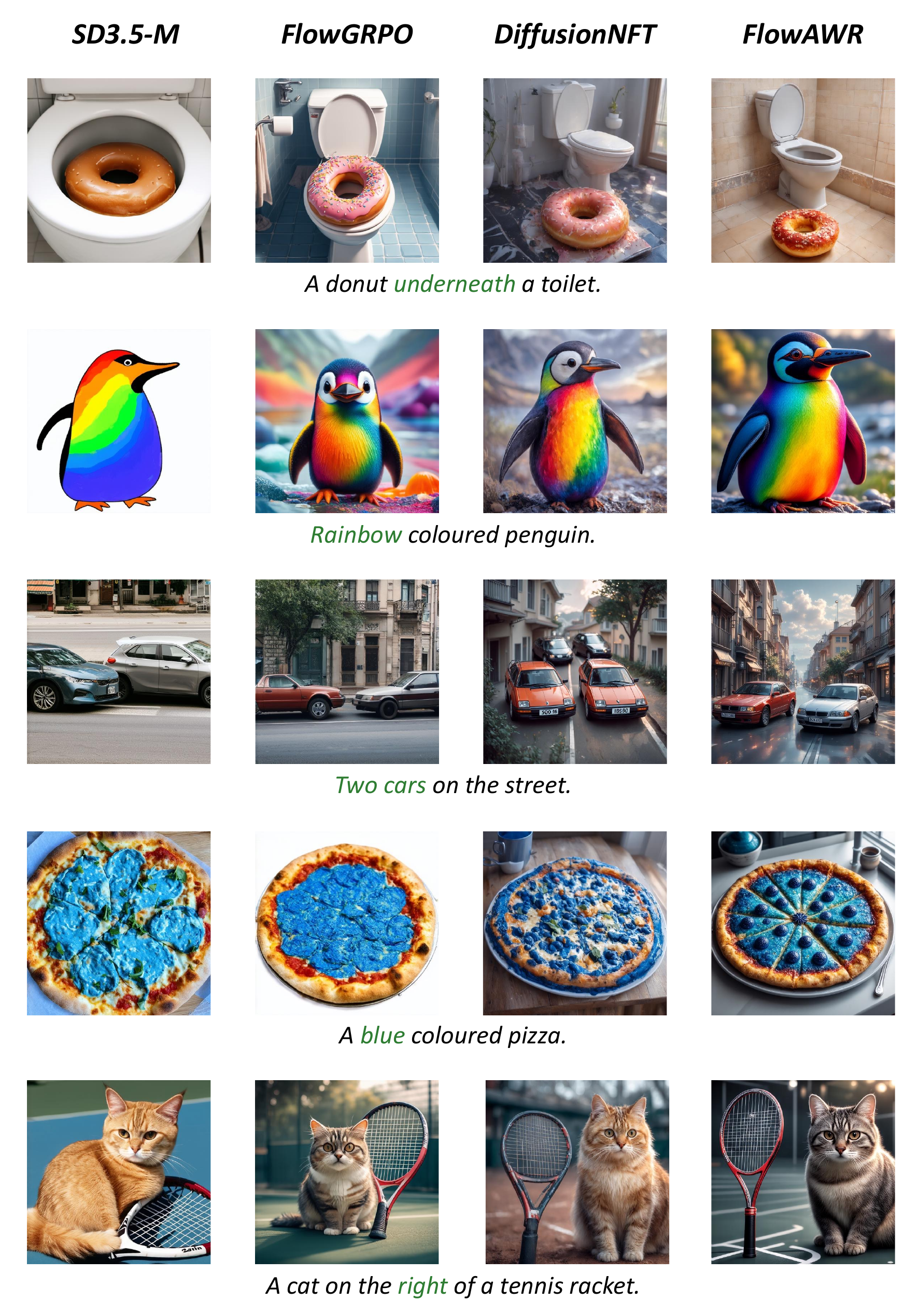}
    \vspace{-0.1in}
    \caption{Qualitative comparison on the DrawBench benchmark.}
    \label{fig:qual_drawbench}
\end{figure*}

\end{document}